\documentclass[11pt]{article}
\usepackage[top=2.54cm,left=2.54cm,right=2.54cm,bottom=2.54cm]{geometry}
\usepackage{natbib}
\usepackage{xcolor}

\usepackage[utf8]{inputenc}
\usepackage[T1]{fontenc}
\usepackage{microtype}

\usepackage[hidelinks]{hyperref}
\usepackage{url}

\usepackage{amsmath,amssymb,amsthm,mathtools,amsfonts}
\usepackage{bm}

\usepackage{graphicx}
\usepackage{booktabs}
\usepackage{multirow}
\usepackage{array}
\usepackage{caption}
\usepackage{subcaption}

\usepackage{enumitem}

\usepackage{algorithm}
\usepackage{algpseudocode}

\usepackage{placeins}

\usepackage{xcolor}
\usepackage{nicefrac}
\usepackage{siunitx}
\usepackage{cleveref}

\usepackage{needspace}

\AtBeginDocument{
  \setlength{\abovedisplayskip}{7pt plus 2pt minus 4pt}%
  \setlength{\belowdisplayskip}{7pt plus 2pt minus 4pt}%
  \setlength{\abovedisplayshortskip}{5pt plus 2pt minus 3pt}%
  \setlength{\belowdisplayshortskip}{5pt plus 2pt minus 3pt}%
  \setlength{\textfloatsep}{10pt plus 2pt minus 3pt}%
  \setlength{\floatsep}{8pt plus 2pt minus 2pt}%
  \setlength{\intextsep}{10pt plus 2pt minus 3pt}%
}
\theoremstyle{plain}
\newtheorem{theorem}{Theorem}
\newtheorem{lemma}{Lemma}[section]

\theoremstyle{definition}
\newtheorem{definition}{Definition}

\theoremstyle{remark}

\usepackage{authblk}
\title{Provably Efficient Federated Reinforcement Learning with Linear Function Approximation and Logarithmic Communication Cost}

\author{Zihang Liang}
\author{Haochen Zhang}
\author{Lingzhou Xue\thanks{Zihang Liang and Haochen Zhang are co-first authors who contributed equally to this paper. Lingzhou Xue is the corresponding author (Email: \texttt{lzxue@psu.edu}).}}
\date{}
\affil{Department of Statistics, The Pennsylvania State University}

\begin{document}

\maketitle

\begin{abstract}
We study federated online reinforcement learning  with linear function approximation. While recent multi-agent reinforcement learning algorithms achieve strong regret guarantees, they typically require sharing raw trajectories. This reliance incurs a communication cost that scales linearly with the number of episodes and violates the privacy constraints of federated settings. To address these limitations, we propose Fed-LSVI, the first provably efficient federated algorithm for online reinforcement learning with linear function approximation in episodic Markov decision processes. By integrating a determinant-based event-triggered synchronization with a stepwise backward update mechanism, Fed-LSVI enables agents to collaboratively learn an optimal policy by exchanging only compressed sufficient statistics. We prove that Fed-LSVI achieves a regret bound of $\widetilde{\mathcal O}(\sqrt{Md^3H^4T})$, where $d$ is the feature dimension, $H$ is the horizon length, $M$ is the number of agents, and $T$ is the number of episodes per agent, matching the best-known regret for multi-agent online reinforcement learning with linear function approximation. Moreover, by following the stringent communication and privacy constraints of the federated setting, Fed-LSVI reduces the communication cost to only logarithmic dependence on $T$, representing a significant improvement over prior methods.

\end{abstract}

\section{Introduction}\label{sec:intro}
Reinforcement learning (RL) studies sequential decision-making in unknown environments and is typically formulated as learning an optimal policy within a Markov decision process (MDP) \citep{sutton2018reinforcement}. In modern RL, designing sample-efficient algorithms for problems with large state and action spaces remains a central challenge. A widely adopted approach is function approximation, which enables scalable learning by representing value functions within a restricted function class.

To build a rigorous foundation for this paradigm, recent work has focused on establishing provable regret guarantees for online RL with linear function approximation, where value functions are modeled as linear functions of known features. This literature is largely anchored by two prominent frameworks. The first is the linear MDP setting \citep{jin2020provably,he2023nearly},  which assumes that both transition dynamics and reward functions admit linear representations in terms of known features. The second is the linear mixture MDP setting \citep{ayoub2020model, zhou2021nearly, zhou2022computationally}, which instead models transitions as a mixture over a set of base kernels.

In many real-world applications, the data collection capacity of a single agent is a primary bottleneck. This motivates cooperative multi-agent reinforcement learning (MARL), where multiple agents collaboratively explore a shared environment and jointly learn an optimal policy. By sharing exploration across agents, cooperative MARL can improve sample efficiency. Theoretical foundations for MARL have recently begun to emerge in the literature. Notably, \citet{dubey2021cooperative} proposed the first provably efficient cooperative MARL algorithm with function approximation, extending the single-agent LSVI-UCB framework \citep{jin2020provably} to demonstrate that parallel exploration can be achieved with low communication overhead. Building on this work, \citet{min2023cooperative} studied cooperative linear MDPs with asynchronous communication via a central server and established both regret and communication guarantees. 
More recently, \citet{hsu2024randomized} introduced randomized exploration strategies with provable guarantees for broader function classes.

However, these cooperative approaches typically require agents to share raw trajectories for collaborative policy learning. This reliance introduces substantial communication overhead and privacy concerns, particularly in real-world applications where sequential decision-making is distributed across institutions or personal devices. In healthcare, for example, RL is increasingly used to optimize dynamic treatment regimes \citep{murphy2003optimal}, where treatment decisions are sequentially adapted to patient histories, such as sepsis management in intensive care units \citep{komorowski2018artificial}. In such settings, because patient-level trajectories are highly sensitive and often distributed across different hospitals or clinical sites, transmitting raw clinical data to a central learner is often rendered impractical by privacy regulations and institutional constraints. Similar bottlenecks exist in mobile edge computing and recommender systems, where local user interaction histories are generated locally and may contain sensitive behavioral patterns \citep{kairouz2021advances}. To circumvent these barriers, it is critical to develop collaborative RL algorithms that do not rely on raw data exchange. Federated learning (FL) \citep{mcmahan2017communication} provides a natural paradigm to achieve this goal: agents keep raw data locally and only transmit aggregated summary statistics or model updates to a central server. This naturally raises the following question:

\begin{center}
\textit{Can we design federated algorithms for online RL with linear function approximation that are simultaneously sample-efficient and communication-efficient?}
\end{center}
Designing such federated algorithms presents two primary technical challenges. First, to satisfy the privacy and communication constraints of the federated setting, agents cannot transmit their local raw trajectories to the server. Instead, the algorithm must identify compressed sufficient statistics that allow the central server to perform policy and value-function updates. While this is straightforward in tabular MDPs through statistics such as state-action visitation counts \citep{zheng2023fedq,zheng2024federated}, online RL with linear function approximation makes this process more challenging, and constructing the corresponding statistics becomes non-trivial. Second, and more fundamentally, classical LSVI-type algorithms rely on a backward update structure: the regression response at step \(h\) depends on the updated value-function estimate at step \(h+1\). Because the central server lacks access to local raw trajectories, a conventional one-shot communication scheme would force the step-\(h\) update to rely on sufficient statistics computed using stale value estimates at step $h+1$. This staleness introduces a severe bias term, which is difficult to control in the regret analysis and may even destroy the sublinear regret guarantee. Therefore, designing a successful federated LSVI algorithm requires a carefully designed synchronization mechanism to address this stale-response issue.

In this work, we provide an affirmative answer to this important open question. We propose the first federated reinforcement learning (FRL) algorithm for online RL with linear function approximation, and our algorithm achieves \emph{sublinear regret} while incurring a communication cost that scales only \emph{logarithmically} with the total number of episodes. Our results show that, even when raw trajectories remain strictly local, agents can collaboratively achieve statistically efficient exploration without suffering from a communication bottleneck. Our main contributions are summarized as follows:

\textbf{Algorithm Design.}
We propose Fed-LSVI, the first provably efficient federated algorithm for online RL with linear function approximation in the literature. Our algorithm relies on two core components: a \emph{determinant-based event-triggered synchronization} mechanism and a \emph{stepwise backward synchronization} protocol. The event-triggered mechanism prevents redundant communication, bounding the total communication rounds logarithmically. Meanwhile, the stepwise backward protocol introduces a dedicated solution to address the stale-response issue in a federated LSVI-type algorithm. Rather than aggregating all response statistics in a single batch, synchronization proceeds backward through the steps. Once the server updates and broadcasts the step-\((h+1)\) model, agents use this latest global value function to construct the step-\(h\) summary statistics, which are then transmitted to the server for the step-\(h\) model update. This carefully designed stepwise backward procedure keeps each regression target consistent with the most up-to-date global value estimate, achieving statistically efficient exploration while maintaining local trajectories private.

\textbf{Regret Guarantees.}
We establish regret upper bounds for Fed-LSVI in both homogeneous and misspecified settings, matching the best-known regret scaling in online MARL with linear function approximation \citep{min2023cooperative,hsu2024randomized}. In the homogeneous setting, Fed-LSVI achieves a regret bound of  $\widetilde{\mathcal{O}}(\sqrt{Md^3H^4T})$ (\Cref{theo_1}), confirming that our communication-efficient protocol does not compromise statistical efficiency. Furthermore, we extend our analysis to a misspecified setting with agent-wise heterogeneity (\Cref{theo_hetero}), where each agent interacts with a $\xi$-approximate linear MDP. We show that Fed-LSVI remains robust to mild agent-wise heterogeneity. In particular, when the misspecification level satisfies $\xi=\widetilde{\mathcal{O}}(\sqrt{d/(MT)})$, our algorithm recovers the same regret bound as in the homogeneous setting.

\textbf{Communication Cost.}
We prove that Fed-LSVI incurs a communication cost, measured by the total number of transmitted scalars, that scales only logarithmically with \(T\) in both the homogeneous and misspecified settings. This logarithmic dependence is driven by two properties: Our event-triggered mechanism ensures that the total number of communication rounds is logarithmic in \(T\) (\Cref{theo_2}), and each round requires transmitting only \(\mathcal{O}(Md^2H)\) scalars between the agents and the server. This represents an exponential improvement over prior cooperative RL algorithms\citep{dubey2021cooperative,min2023cooperative,hsu2024randomized}, which require sharing local raw trajectories and then suffer from communication costs that scale linearly with \(T\). Thus, Fed-LSVI significantly reduces communication overhead without transmitting raw trajectories, strictly respecting the privacy constraints inherent in federated learning.

\section{Related Work}
\label{sec:related_work}

\textbf{Near-Optimal Reinforcement Learning.}
In the single-agent setting, the literature most relevant to our work spans near-optimal tabular RL and RL with linear function approximation. In tabular RL, algorithms are typically categorized into model-based and model-free approaches. A large body of work has focused on model-based algorithms \citep{agarwal2020model, agrawal2017optimistic, auer2008near, azar2017minimax, dann2019policy, zanette2019tighter, zhang2024settling, zhang2021reinforcement, zhou2023sharp}. Notably, \citet{zhang2024settling} proposed an algorithm achieving a regret bound of $\tilde{O}(\min\{\sqrt{SAH^2T},\, T\})$, which matches the information-theoretic lower bound. There is also a large body of work on model-free methods that has progressively narrowed the theoretical gap, including~\cite{jin2018q, li2023breaking, menard2021ucb, yang2021q, zhang2020almost,zheng2025gap,zhang2025regret,zhang2026q}. Several works~\citep{zhang2020almost, menard2021ucb,li2023breaking, zhang2025regret} achieved the near-optimal regret bound $\tilde{O}(\sqrt{SAH^2T})$.

The literature on RL with linear function approximation can in turn be divided according to the structural assumptions imposed on the MDP. One line studies linear MDPs \citep{yang2019sample, jin2020provably,wei2021learning, wagenmaker2022first, he2023nearly,zhanghorizon,zhang2026gap}. Foundational works established sample-efficient learning under generative-model or feature-space assumptions \citep{yang2019sample,yang2020feature}, and \citet{jin2020provably} later introduced LSVI-UCB, the first provably efficient online algorithm without access to a generative model. Subsequent works substantially sharpened the regret guarantees for linear MDPs, culminating in nearly minimax-optimal results \citep{hu2022nearly,he2023nearly}. Another line studies linear-mixture MDPs \citep{jia2020model, ayoub2020model, modi2020sample, zhou2021nearly, zhou2021provably, zhou2022computationally}. The near-minimax-optimal regret guarantees are also available \citep{zhou2021nearly}. Furthermore, \citet{zhou2022computationally} proposed a near-optimal, horizon-free algorithm for time-homogeneous linear mixture MDPs. \citet{zhang2023optimal} provided the near-optimal horizon-free sample complexity in the reward-free time-homogeneous setting.

 \textbf{Low-Switching Reinforcement Learning.} A closely related direction studies learning with limited switching or adaptivity, motivated by the cost of frequent policy deployment.  In the tabular setting, \citet{bai2019provably} established $Q$-learning with low switching cost, \citet{qiao2022sample} characterized near-optimal regret--switching trade-offs up to logarithmic factors, and \citet{velegkas2022reinforcement} showed that positive-gap instances permit both logarithmic regret and logarithmically many policy switches. Under linear function approximation, \citet{wang2021provably} proved that rare policy updates can still be combined with statistically efficient learning.

\textbf{Multi-Agent RL (MARL) with Event-Triggered Communications.}
In the multi-agent setting, the literature most relevant to our work can be divided into two directions: cooperative or parallel exploration with a shared objective, and decentralized equilibrium learning in Markov games under function approximation. The first direction is closest to our setting. \citet{dubey2021cooperative} initiated provably efficient cooperative multi-agent reinforcement learning with function approximation and showed that parallel exploration can be combined with limited inter-agent communication. Building on this line, \citet{min2023cooperative} studied cooperative linear MDPs with asynchronous communication through a central server and established regret and communication guarantees, while \citet{hsu2024randomized} introduced randomized exploration in cooperative multi-agent reinforcement learning and analyzed its performance under approximately linear transitions.

\textbf{Federated Reinforcement Learning.}
Recent finite-time analyses of federated reinforcement learning have mainly focused on tabular control, policy evaluation, or policy-gradient methods. In tabular value-based RL, prior works established linear speedup, communication efficiency, and heterogeneity-dependent guarantees for federated Q-learning and related methods \citep{khodadadian2022federated,woo2023blessing,zheng2023fedq,zheng2024federated,salgia2024sample,labbi2024feducbvi,zhang2025gapdependent,zhang2025regret}. A complementary line studies federated policy evaluation or on-policy learning with function approximation, including TD-learning under Markovian sampling and heterogeneous environments \citep{dalfabbro2023fedtd,wang2024fedtd,zhang2024fedsarsa}, as well as federated policy-gradient and actor-critic methods \citep{yang2024fednpg,ganesh2024fedpg,lan2025asynchronous}.

\section{Preliminaries}\label{sec:preliminary}

\paragraph{Episodic Markov Decision Processes.}
In this paper, we consider a finite-horizon episodic Markov decision process
$\mathcal M=(\mathcal X,\mathcal A,H,\{\mathbb{P}_h\}_{h=1}^H,\{r_h\}_{h=1}^H)$,
where $\mathcal X$ is a measurable, possibly infinite state space, $\mathcal A$ is a finite action set, $H$ is the horizon length, each deterministic reward function satisfies
$r_h:\mathcal X\times\mathcal A\to[0,1]$, and
$\mathbb{P}_h(\cdot\mid x,a)$ denotes the transition kernel at step $h\in[H]$. Here, For any $C\in \mathbb{N}$, we use $[C]$ to denote the set $\{1,2,\ldots C\}$.
In each episode, the agent starts from an initial state $x_1$. At each step $h \in [H]$, it observes a state $x_h$, chooses an action $a_h\in\mathcal A$, receives reward $r_h(x_h,a_h)$, and transitions to
$x_{h+1}\sim \mathbb{P}_h(\cdot\mid x_h,a_h)$ for each $h\in[H]$. 
\paragraph{Policies, Value Functions, and Bellman Equations.}  
A policy is a collection $\pi=\{\pi_h\}_{h=1}^H$ with
$\pi_h:\mathcal X\to\Delta(\mathcal A)$, where $\Delta(\mathcal A)$ denotes the set of probability distributions over $\mathcal A$.

To evaluate the performance of different policies, we define the $Q$-value function and $V$-value function under a policy $\pi$ as follows: for any $(x,a) \in \mathcal{X} \times \mathcal{A}$, 

\[Q^{\pi}_h(x,a) \coloneqq \mathbb E_{\pi}\left[\sum_{h' = h}^H r_{h'}(x_{h'},a_{h'})| x_h = x,a_h = a\right],\ V^{\pi}_h(x) = \mathbb E_{\pi}\left[\sum_{h'=h}^Hr_{h'}(x_{h'},a_{h'})|x_h = x\right].\]
For any bounded function $V:\mathcal X\to\mathbb R$, define $$[\mathbb P_hV](x,a)=\mathbb E_{x'\sim \mathbb{P}_h(\cdot\mid x,a)}[V(x')].$$

Since the action space and the horizon are all finite, there exists an optimal policy $\pi^{\star}$ that achieves the optimal value $V_h^{\star}(x)=\sup _\pi V_h^\pi(x)=V_h^{\pi^\star}(x)$ for all $(x,h) \in \mathcal{X} \times [H]$  \citep{azar2017minimax}. For any $(x, a, h) \in \mathcal{X} \times \mathcal{A} \times [H]$, the Bellman equation can be expressed as ($V_{H+1}^{\pi}(\cdot) \equiv V_{H+1}^{\star}(\cdot) \equiv 0$)
    \begin{equation}\label{eq_Bellman}
	\begin{aligned}
		&\left\{
		\begin{array}{l}
			V_h^{\pi}(x) = \mathbb{E}_{a' \sim \pi_h(x)}[Q_h^{\pi}(x, a')] \\
			Q_h^{\pi}(x, a) := r_h(x, a) + [\mathbb P_hV_{h+1}^{\pi}](x,a),
		\end{array}
		\right. 
		&\left\{
		\begin{array}{l}
			V_h^{\star}(x) = \max_{a' \in \mathcal{A}} Q_h^{\star}(x, a') \\
			Q_h^{\star}(x, a) := r_h(x, a) + [\mathbb P_hV_{h+1}^{\star}](x,a).
		\end{array}
		\right.
	\end{aligned}
\end{equation}

\paragraph{Linear Markov Decision Processes.}
Let $\phi:\mathcal{X}\times\mathcal{A}\to\mathbb{R}^d$ be a known feature map satisfying $\|\phi(x,a)\|_2\le 1$ for all $(x,a)$.
We adopt the standard linear MDP assumption from the literature on linear function approximation \citep{yang2020feature,jin2020provably,wang2021provably}.
\begin{definition}[Linear MDPs]\label{ass:linear_mdp}
For each step $h\in[H]$, there exist an unknown vector $\theta_h\in\mathbb{R}^d$ and an unknown vector of signed measures
$\mu_h=(\mu_h^{(1)},\ldots,\mu_h^{(d)})$ over $\mathcal{X}$ such that, for every $(x,a)\in\mathcal{X}\times\mathcal{A}$ and every measurable set $\mathcal B\subseteq \mathcal X$,
\[
    r_h(x,a)=\langle \phi(x,a),\theta_h\rangle,
    \qquad
    \mathbb{P}_h(\mathcal B\mid x,a)=\langle \phi(x,a),\mu_h(\mathcal B)\rangle.
\]
Moreover, $\|\theta_h\|_2\le \sqrt d$ and $\|\mu_h(\mathcal X)\|\le \sqrt d$ for all $h\in[H]$.
\end{definition}

\paragraph{Federated Reinforcement Learning.}
We consider an FRL setting with a central server and $M$ agents, each interacting with an independent linear MDP $\mathcal{M}$. The paper studies a synchronous protocol: at common episode index $t\in[T]$, each agent executes its $t$-th episode, starting from an initial state $x_1^{m,t}$ that may be arbitrary.
Thus $T$ is the number of episodes per agent, the total number of agent-episodes is $MT$, and the total number of interaction steps is $MTH$. In episode $t$, agent $m\in[M]$ generates a trajectory $\{(x_h^{m,t},a_h^{m,t},r_h^{m,t},x_{h+1}^{m,t})\}_{h=1}^H$. A central server coordinates occasional synchronizations, while raw trajectories always remain on the agents.

Specifically, we partition the learning process into $K$ rounds with boundaries $1 = \tau_1 < \tau_2 < \cdots < \tau_{K+1} = T+1,$ where round $k$ consists of episodes $\tau_k, \ldots, \tau_{k+1}-1$. All agents execute the same policy $\pi^k$ throughout round $k$. The cumulative regret over $T$ synchronized episodes is defined as
\[
\mathrm{Regret}(T)
:=
\sum_{t=1}^T\sum_{m=1}^M
\Bigl(
V_{m,1}^{\star}(x_1^{m,t})-V_{m,1}^{\pi^{k(t)}}(x_1^{m,t})
\Bigr),
\]
where $V_{m,1}^{\star}$ and $V_{m,1}^{\pi}$ denote the optimal value function and $V$-value function of agent $m$'s local MDP and $k(t)$ denotes the round containing episode $t$. In the homogeneous setting, where all agents share the same MDP, we omit the subscript $m$ and write $V_{1}^{\star}$ and $V_{1}^{\pi}$ for simplicity.

We also define the communication cost of an FRL algorithm as the number of scalars (integers or real numbers) communicated between the server and agents.

\section{Algorithm Design}\label{sec:algorithm_design}

\subsection{The Fed-LSVI Algorithm}
\vspace{-0.15em}
We now present Fed-LSVI under the federated reinforcement learning setting described in \Cref{sec:preliminary}. The complete algorithm is given in \Cref{alg:server_sync} and \Cref{alg:agent_sync}. At the beginning of learning, the central server specifies the total number of episodes per agent $T$, the regularization parameter $\lambda$, the trigger parameter $\gamma \geq 1$, and the exploration bonus parameter $\beta>0$. Fed-LSVI then proceeds over communication rounds indexed by $k \in \{1,2,\ldots,K\}$.

\begin{algorithm}[ht]
\caption{Central server}
\label{alg:server_sync}
\small
\begin{algorithmic}[1]
\State \textbf{Input:} regularizer $\lambda>0$, trigger parameter $\gamma\ge 1$, bonus scale $\beta>0$, hard episode budget $T$.
\State \textbf{Init:} for each $h\in[H]$, set $\Lambda_h^{\mathrm{ser}}\gets \lambda I_d$ and $w_h\gets 0$; set the common episode counter $t\gets 0$.
\State Broadcast parameters $\{w_h,\Lambda_h^{\mathrm{ser}},\gamma,\beta, T\}$ to all agents.
\While{$t<T$}  

    \State \textbf{Wait} until receiving a synchronization signal 
    and set \(t\) to the received current episode counter.

    \State \textbf{Broadcast} the synchronization signal to all agents.
    \For{$h=H,H-1,\dots,1$}
        \State Receive $\{(\Lambda_h^{m,\mathrm{loc}},b_h^m)\}_{m=1}^M$ from all agents.
        \State  $\Lambda_h^{\mathrm{ser}}\gets \Lambda_h^{\mathrm{ser}}+\sum_{m=1}^M\Lambda_h^{m,\mathrm{loc}}$.
         \State $b_h\gets \sum_{m=1}^M b_h^m$.
        \State Update $w_h\gets (\Lambda_h^{\mathrm{ser}})^{-1}b_h$.
        \State \textbf{Broadcast} $(w_h,\Lambda_h^{\mathrm{ser}})$ to all agents.
    \EndFor
\EndWhile
\end{algorithmic}
\end{algorithm}

\paragraph{Global Coordination.}
At the beginning of round $k$, each agent holds the current global model $\{(w_h,\Lambda_h^{\mathrm{ser}})\}_{h=1}^H$. In the special case $k=1$, this model is initialized and broadcast by the central server before data
collection begins. For $k\ge 2$, it is the model updated and broadcast by the central server at the end of round $k-1$. Here, $w_h\in\mathbb R^d$ is the regression parameter and $\Lambda_h^{\mathrm{ser}}\in\mathbb R^{d\times d}$ is the server-side covariance matrix at step $h$.

\begin{algorithm}[ht]
\caption{Agent $m$}
\label{alg:agent_sync}
\small
\begin{algorithmic}[1]

\State \textbf{Init:} Set $\mathcal D_h^m\gets\emptyset$ for all $h\in[H]$ and $t = 0$. 
\State Receive parameters $\{w_h,\Lambda_h^{\mathrm{ser}},\gamma,\beta, T\}$ from the central server and then initialize $Q_h$ and $V_h$ by \eqref{eq:QV}.
\While{$t<T$}

    \State Set the round length counter $\Delta t\gets 0$.
    \For{$h=1,2,\dots,H$}
        \State $\Lambda_h^{m,\mathrm{loc}}\gets 0$.
    \EndFor

    \While{no synchronization signal has been received and $t<T$}
        \State $t\gets t+1$ and $\Delta t\gets \Delta t+1$.
        \State Observe the initial state $x_1^{m,t}$.
        \For{$h=1,2,\dots,H$}
            \State Evaluate the latest optimistic model and pick $a_h^{m,t}\gets \arg\max_a Q_h(x_h^{m,t},a)$.
            \State Execute $a_h^{m,t}$, observe reward $r_h^{m,t}$ and next state $x_{h+1}^{m,t}$.
            \State Append $(x_h^{m,t},a_h^{m,t},r_h^{m,t},x_{h+1}^{m,t})$ to $\mathcal D_h^m$.
            \State $\Lambda_h^{m,\mathrm{loc}}\gets \Lambda_h^{m,\mathrm{loc}}+\phi(x_h^{m,t},a_h^{m,t})\phi(x_h^{m,t},a_h^{m,t})^\top$.
        \EndFor
        \If{$t < T$ and $\exists h\in[H]$ such that
        $\dfrac{\det(\Lambda_h^{\mathrm{ser}}+\Lambda_h^{m,\mathrm{loc}})}{\det(\Lambda_h^{\mathrm{ser}})}\ge \dfrac{\gamma}{\Delta t}$}
            \State Send synchronization signal and current episode counter $t$ to the server.
        \EndIf
    \EndWhile
    \If{$t = T$}
    \State Send synchronization signal and current episode counter $t$ to the server. 
\EndIf

    \State \textbf{Synchronization (backward, stepwise):} set $V_{H+1}(\cdot)\equiv 0$.
    \For{$h=H,H-1,\dots,1$}
        \State Compute the full-history target $b_h^m$ as \eqref{eq:b}.
        \State Send $(\Lambda_h^{m,\mathrm{loc}},b_h^m)$ to the server.
        \State Receive the updated step-$h$ model $(w_h,\Lambda_h^{\mathrm{ser}})$ from the server.
        \State Update $Q_h$ and $V_h$ according to \eqref{eq:QV}.
    \EndFor
\EndWhile
\end{algorithmic}
\end{algorithm}

\paragraph{Local Exploration.} Using the known feature mapping $\phi(x,a)$ together with the broadcast global
model $\{(w_h,\Lambda_h^{\mathrm{ser}})\}_{h=1}^H$, for any $(x,h) \in \mathcal{X} \times [H]$, each agent can reconstruct the optimistic value estimates $Q_h(x,a)$, $V_h(x)$ according to
\begin{equation}
\label{eq:QV}
Q_h(x,a)
=\Pi_{[0,H]}\left(\phi(x,a)^\top w_h+\beta\,\|\phi(x,a)\|_{(\Lambda_h^{\mathrm{ser}})^{-1}}\right).
\quad
V_h(x)=\max_{a\in\mathcal A}Q_h(x,a),
\end{equation}
Here, $\Pi_{[0,H]} := \min\{\max\{z,0\},H\}$. This projection ensures $0\leq Q_h(x,a)\leq H$ and hence $0\leq V_h(x)\leq H$. The policy at $(x,h)$ is then defined by 
$$\pi_h(x)=\arg\max_{a\in\mathcal A}Q_h(x,a).$$
During round $k$, all agents execute the policy $\pi$ and repeatedly
collect trajectories over multiple episodes. Each agent maintains and updates a local copy of the common episode counter $t$, together with a current-round episode counter $\Delta t$, which records the number of episodes collected since the beginning of round $k$. Under the synchronous protocol, the local copies of $t$ remain identical across agents. We
denote the trajectory collected by agent $m$ in episode $t$ as
$\{(x_h^{m,t},a_h^{m,t},r_h^{m,t},x_{h+1}^{m,t})\}_{h=1}^H$.

During data collection, each agent $m$ maintains two types of local information for each step $h\in[H]$: a persistent buffer $\mathcal D_h^m$ and a round-wise local covariance matrix $\Lambda_h^{m,\mathrm{loc}}$. The buffer $\mathcal D_h^m$ stores all historical transition tuples $(x_h^{m,t},a_h^{m,t},r_h^{m,t},x_{h+1}^{m,t})$ observed at step $h$ throughout the learning process. In contrast, during each round $k$, the matrix $\Lambda_h^{m,\mathrm{loc}}$ accumulates only the step-$h$ covariance increments from episodes collected in that round, namely $\phi(x_h^{m,t},a_h^{m,t})\phi(x_h^{m,t},a_h^{m,t})^\top$ for each such episode $t$, and is reset to zero at the beginning of the next round.
\paragraph{Determinant-Based Event-Triggered Synchronization.} In round $k$, synchronization is triggered if there exists an agent $m \in [M]$ such that one of the following two events occurs:  
(i) the determinant-based trigger condition is satisfied, namely, for some $h \in [H]$,
\begin{equation}
\label{trigger}
\frac{\det\!\big(\Lambda_h^{\mathrm{ser}}+\Lambda_h^{m,\mathrm{loc}}\big)}
{\det\!\big(\Lambda_h^{\mathrm{ser}}\big)}
\ge \frac{\gamma}{\Delta t};
\end{equation}
(ii) the total episode counter reaches the overall episode budget, that is, $t = T$.

In either case, the triggering agent $m$ sends a synchronization signal together with the current episode counter $t$ to the central server, which then broadcasts the signal to all agents and terminates the exploration in round $k$.

\paragraph{Agent-Server Synchronization and Policy Update.}
At the end of each round, the central server updates the value-function estimates and the policy through communication with the agents. To update the step-$h$ value-function estimate, similar to LSVI-UCB \citep{jin2020provably}, the central server solves a ridge regression problem whose targets are labeled by the next-step value function $V_{h+1}$:
\begin{align}
\label{eq:ridge}
w_h&=\arg\min_{w\in \mathbb R^d}\sum_{m=1}^M\sum_{(x,a,r,x')\in \mathcal D_h^m}\bigl(w^\top \phi(x,a)-r-V_{h+1}(x')\bigr)^2+\lambda\|w\|_2^2 .
\end{align}
The resulting estimator $w_h$ is then used to update the value-function estimates and the policy at step $h$. Importantly, the value function $V_{h+1}$ used in this regression should be the newly updated one, computed using all historical samples, including those collected in the current round. Therefore, the updates must be carried out backward over the horizon.

This motivates an interleaved, \emph{stepwise backward} synchronization procedure. Specifically, at each step, the agents first send the necessary statistics to the central server. The server then immediately updates the corresponding value-function and policy estimates and broadcasts the updated stepwise model back to the agents before proceeding to the previous step.
 
Specifically, to solve the above regression problem, each agent $m$ forms the full-history response vector $b_h^m$ from its local buffer $\mathcal D_h^m$:
\begin{equation}
\label{eq:b}
b_h^m
=
\sum_{(x,a,r,x')\in \mathcal D_h^m}
\phi(x,a)\bigl(r+V_{h+1}(x')\bigr),
\end{equation}
and sends it, together with its local covariance matrix $\Lambda_h^{m,\mathrm{loc}}$, to the central server. The server aggregates the uploaded statistics $\{(\Lambda_h^{m,\mathrm{loc}}, b_h^m)\}_{m=1}^M$ and updates the global covariance matrix $\Lambda_h^{\mathrm{ser}}$ and global response vector $b_h$ according to
\begin{equation}
\label{lambdab}
\Lambda_h^{\mathrm{ser}}
\leftarrow
\Lambda_h^{\mathrm{ser}}+\sum_{m=1}^M \Lambda_h^{m,\mathrm{loc}},
\quad
b_h
\leftarrow
\sum_{m=1}^M b_h^m,
\end{equation}
and then updates the regression coefficient $w_h$ by solving the ridge regression problem in \eqref{eq:ridge}, which admits the closed-form solution:
\begin{equation}
w_h
=
(\Lambda_h^{\mathrm{ser}})^{-1} b_h .
\label{eq:update}
\end{equation}
The central server broadcasts the updated step-$h$ model $(w_h,\Lambda_h^{\mathrm{ser}})$ to all agents before proceeding to step $h-1$. Using this newly updated step-$h$ model, all agents can construct the same value function $V_h$ with \eqref{eq:QV}, which is then used to form the response vectors $b_{h-1}^m$ by \eqref{eq:b} for the model update at step $h-1$. In this way, synchronization and model updating are carried out recursively from step $H$ down to step $1$. This backward, stepwise order ensures that all agents label their historical step-$h$ samples using the same freshly updated value function $V_{h+1}$ as in LSVI-UCB.

\section{Theoretical Guarantees}\label{sec:theoretical_analysis}
\vspace{-0.1em}
\subsection{Homogeneous Setting}
We first consider the homogeneous setting, where all agents interact with identical copies of the same linear MDP. In this setting, we present the regret and communication guarantees of Fed-LSVI.

\begin{theorem}[Regret in the Homogeneous Setting]
\label{theo_1}
For Algorithms~\ref{alg:server_sync} and~\ref{alg:agent_sync} with trigger parameter $\gamma\ge 1$ and regularization parameter $\lambda=1$, there exists a positive constant $c_\beta>0$ such that, for any confidence level $p\in(0,1)$, if we choose $\beta = c_\beta Hd\sqrt{\log(1+dMHT/p)}$, then with probability at least $1-2p$, 

\[
\mathrm{Regret}(T)
\;\le\;
\widetilde{\mathcal O}\!\left(\sqrt{Md^3H^4T}+MH^2d^2\sqrt{\gamma}\right).
\]
The complete expression is provided in \eqref{finalregret} of Appendix \ref{app:regret_analysis}.
\end{theorem}
To the best of our knowledge, this is the first regret guarantee for online federated reinforcement learning under linear function approximation. The proof of \Cref{theo_1} is provided in \Cref{app:regret_analysis}. Compared with best-known multi-agent results under linear function approximation \citep{min2023cooperative, hsu2024randomized}, our guarantee matches their $\widetilde{\mathcal O}(\sqrt{Md^3H^4T})$ regret scaling while being tailored to the stricter federated regime, where raw trajectories are not shared and only stage-wise sufficient statistics are communicated.

When $T \geq Md$ and choose $\gamma = T/Md$, Theorem~\ref{theo_1} implies $\mathrm{Regret}(T)\le\widetilde{\mathcal O}(\sqrt{Md^{3}H^4T}).$ Therefore, the average regret per agent is of order $\widetilde{\mathcal O}(\sqrt{d^{3}H^4T/M})$, exhibiting a $\sqrt{M}$-type multi-agent speedup, consistent with worst-case results in federated tabular RL \citep{zheng2023fedq, zheng2024federated, labbi2024feducbvi, zhang2025regret} and in multi-agent RL with linear function approximation \citep{dubey2021cooperative, min2023cooperative, hsu2024randomized}.

\subsection{Misspecified Setting}
We next consider a heterogeneous extension under the misspecified setting, where agents interact with agent-specific MDPs, each of which is $\xi$-approximate to a common reference linear MDP. We show that the proposed algorithm is robust to such mild agent-wise heterogeneity.

\begin{definition}[Misspecified Setting]
\label{ass:small_heterogeneity}
For any $0 < \xi \le 1$, and for any agent $m \in [M]$, the corresponding
MDP $(\mathcal{X}, \mathcal{A}, H, \mathbb{P}_m, r_m)$ is a $\xi$-approximate linear
MDP with a feature map $\phi : \mathcal{X} \times \mathcal{A} \to \mathbb{R}^d$. Specifically, for each step $h \in [H]$, there exist $d$ unknown measures 
$\mu_h = \bigl(\mu_h^{(1)}, \ldots, \mu_h^{(d)}\bigr)$ over $\mathcal{X}$ and an unknown vector 
$\theta_h \in \mathbb{R}^d$ such that, for all $(x,a) \in \mathcal{X} \times \mathcal{A}$,

\[
\bigl\|\mathbb{P}_{m,h}(\cdot \mid x,a) - \langle \phi(x,a), \mu_h(\cdot)\rangle\bigr\|_{\mathrm{TV}} \le \xi,\quad
\bigl|r_{m,h}(x,a) - \langle \phi(x,a), \theta_h\rangle\bigr| \le \xi.
\]
Here, we assume $\|\phi(x,a)\| \le 1$ for all
$(x,a) \in \mathcal{X} \times \mathcal{A}$, and $\max\bigl\{\|\mu_h(\mathcal{X})\|,\|\theta_h\|\bigr\} \le \sqrt{d}$.

\end{definition}
Note that the misspecified setting in \Cref{ass:small_heterogeneity} coincides with the setting considered in \cite{hsu2024randomized}, which generalizes the single-agent misspecified setting in \cite{jin2020provably}. Moreover, it also covers the small heterogeneity setting studied in \cite{dubey2021cooperative}.
 
We now present the regret guarantee under the misspecified setting described in \Cref{ass:small_heterogeneity}.
\begin{theorem}[Regret in the Misspecified Setting]\label{theo_hetero}
Under \Cref{ass:small_heterogeneity}, for Algorithms~\ref{alg:server_sync} and~\ref{alg:agent_sync} with trigger parameter $\gamma\ge 1$ and regularization parameter $\lambda=1$, there exists a positive constant $c_\beta'>0$ such that, for any confidence level $\bar p\in(0,1)$, if we choose $\beta =c_\beta' Hd\sqrt{\log(1+2dMHT/\bar p)}+2H\xi\sqrt{dMT}$, then with probability at least $1-2\bar p$,
\[
\mathrm{Regret}(T)
\;\le\;
\widetilde{\mathcal O}\!\left(\sqrt{Md^3H^4T}+Md^2H^2\sqrt{\gamma}+\xi MdH^2T+ \xi \sqrt{M^3d^3H^4\gamma T}
\right).
\]
The complete expression is provided in \eqref{finalregretmis} of Appendix~\ref{app:hetero_small}.

\end{theorem}

The proof is deferred to Appendix~\ref{app:hetero_small}. This result establishes a regret guarantee for the misspecified setting that matches the best-known multi-agent RL result in \cite{hsu2024randomized}, and demonstrates that the proposed algorithm is robust to bounded agent-wise heterogeneity. In particular, when $\xi = \widetilde{\mathcal O}(\sqrt{d/(MT)})$, the regret bound matches that of the homogeneous setting in \Cref{theo_1}.

\subsection{Communication Cost}
We next present a unified result of the communication cost of Fed-LSVI in both the homogeneous and misspecified settings.
\begin{theorem}\label{theo_2}
For Algorithms~\ref{alg:server_sync} and~\ref{alg:agent_sync}, with $\lambda = 1$ and a total of $T$ episodes per agent, the total number of communication rounds $K$ satisfies

\[
K \;\le\; 1+\frac{2T}{\gamma}+\frac{dH}{\log 2}\log\!\left(1+\frac{MT}{d}\right).
\]
\end{theorem}
The proof is deferred to Appendix \ref{app:proof_comm}. If we choose the trigger parameter as $\gamma=\max\{T/(Md),1\}$, then Theorem~\ref{theo_2} yields
\[
K \;\le\; 1+2Md + \frac{dH}{\log 2}\log\!\left(1+\frac{MT}{d}\right).
\]
Compared with the communication-round bounds of $\tilde{\mathcal O}(dHM)$ in \cite{dubey2021cooperative} and \cite{hsu2024randomized}, and $\tilde{\mathcal O}(dHM^2)$ in \cite{min2023cooperative} for the asynchronous setting, our result improves the dependence on the number of agents $M$ in the leading logarithmic term.

Moreover, in each round, our algorithm communicates only $\mathcal{O}(MHd^2)$ scalars between the agents and the central server. Consequently, the total communication cost of Fed-LSVI is
\[
\mathcal{O}\!\left(M^2d^3H + Md^3H^2 \log\!\left(1+\frac{MT}{d}\right)\right).
\]
In contrast, existing multi-agent online RL methods \citep{dubey2021cooperative, min2023cooperative, hsu2024randomized} typically require sharing raw trajectories from local agents, leading to a communication cost of $\mathcal{O}(MHT)$, which scales linearly with $T$. Therefore, Fed-LSVI reduces the communication cost from linear to logarithmic in $T$, yielding a significant improvement.

When $M=1$, Theorems~\ref{theo_1} to~\ref{theo_2} together imply that our method reduces to a low-switching-cost variant of LSVI-UCB \citep{jin2020provably}, achieving the same $\tilde{\mathcal{O}}(\sqrt{d^3 H^4 T})$ regret while incurring only a logarithmic policy switching cost.

\section{Conclusion}\label{sec:conclusion}
In this paper, we introduced Fed-LSVI, the first provably efficient algorithm for federated online reinforcement learning under linear MDPs. By seamlessly integrating a determinant-based event trigger synchronization rule with a stepwise backward update mechanism, Fed-LSVI achieves both statistical and communication efficiency without sharing raw local trajectories. Theoretically, Fed-LSVI attains \(\sqrt{T}\)-type regret guarantees with a provable multi-agent speedup, while incurring only a logarithmic communication cost. We further extend our analysis to the misspecified setting, showing that Fed-LSVI remains robust to mild agent-wise heterogeneity.Our results establish a principled and communication-efficient framework for optimistic exploration in federated reinforcement learning.

\bibliographystyle{plainnat}
\bibliography{refs}
\appendix

\section{Auxiliary Notation}\label{app:notation}
We collect the notation used throughout the proofs.
The index $m\in[M]$ denotes an agent, $k\in[K]$ denotes a communication round, and $t\in[T]$ denotes the \emph{common} episode index shared by all agents.
The function $k(t)$ returns the index of the round containing episode $t$. For a vector $u,v \in \mathbb{R}^d$ and a positive semi-definite matrix $A \in \mathbb{R}^{d \times d}$, we define $\langle u,v\rangle_A=u^\top A v$ and $\|u\|_{A} := \sqrt{u^\top A u}$. 

To define a single filtration over all agents and episodes, we use the linearization
\[
\ell(m,t)\coloneqq m+M(t-1),
\qquad
\nu_T(\zeta)\coloneqq \Bigl\lceil \frac{\zeta}{M}\Bigr\rceil,
\qquad
\nu_M(\zeta)\coloneqq ((\zeta-1)\bmod M)+1.
\]
Then $(m,t)\mapsto \ell(m,t)$ and $\zeta\mapsto (\nu_M(\zeta),\nu_T(\zeta))$ form a pair of inverse mappings. Accordingly, we define the ordering $(m,t)\le (m',t')$ to mean $\ell(m,t)\le \ell(m',t')$.

For each step $h\in[H]$, let $\Lambda_h^k$ denote the server-side Gram matrix $\Lambda_h^{\mathrm{ser}}$ at the beginning of round $k$, and let $\Lambda_h^{m,k,\mathrm{loc}}$ denote the information increment collected by agent $m$ during round $k$. By the design of the algorithm, these quantities satisfy
\[
\begin{aligned}
&\Lambda_h^k
\coloneqq
\lambda I_d
+\sum_{m=1}^M\sum_{t=1}^{\tau_k-1}
\phi(x_h^{m,t},a_h^{m,t})\phi(x_h^{m,t},a_h^{m,t})^\top,
\\
&\Lambda_h^{m,k,\mathrm{loc}}
\coloneqq
\sum_{t=\tau_k}^{\tau_{k+1}-1}
\phi(x_h^{m,t},a_h^{m,t})\phi(x_h^{m,t},a_h^{m,t})^\top.
\end{aligned}
\]
By the definitions above, for any $k \in [K]$, the quantities $\Lambda_h^{k+1}$, $\Lambda_h^k$, and $\Lambda_h^{m,k,\mathrm{loc}}$ further satisfy
\[
\Lambda_h^{k+1}=\Lambda_h^k+\sum_{m=1}^M\Lambda_h^{m,k,\mathrm{loc}},
\]
where 
\begin{equation}
\label{lastlambda}
    \Lambda_h^{K+1}
\coloneqq
\lambda I_d
+\sum_{m=1}^M\sum_{t=1}^{T}
\phi(x_h^{m,t},a_h^{m,t})\phi(x_h^{m,t},a_h^{m,t})^\top.
\end{equation}
In addition, we define the following three auxiliary statistics, which will be used in the subsequent analysis:
\[
\begin{aligned}
\bar\Lambda_h^{m,t}
&\coloneqq
\Lambda_h^{k(t)}
+\sum_{j=\tau_{k(t)}}^{t}
\phi(x_h^{m,j},a_h^{m,j})\phi(x_h^{m,j},a_h^{m,j})^\top,\ \forall t\in[T]
\\
\tilde\Lambda_h^{m,t}
&\coloneqq
\lambda I_d
+\sum_{(n,j)\le (m,t)}
\phi(x_h^{n,j},a_h^{n,j})\phi(x_h^{n,j},a_h^{n,j})^\top,\ \forall t\in[T]
\\
\check\Lambda_h^t
&\coloneqq
\lambda I_d+\sum_{m=1}^M\sum_{j=1}^{t-1}
\phi(x_h^{m,j},a_h^{m,j})\phi(x_h^{m,j},a_h^{m,j})^\top,\ \forall t\in[T+1].
\end{aligned}
\]
The value of $w_h$ at the beginning of round $k$, denoted by $w_h^k$, is given by
\[
w_h^k=(\Lambda_h^k)^{-1}\sum_{m=1}^M\sum_{t=1}^{\tau_k-1}\phi(x_h^{m,t},a_h^{m,t})\bigl[r_h(x_h^{m,t},a_h^{m,t})+V_{h+1}^{k}(x_{h+1}^{m,t})\bigr].
\]
In the heterogeneous setting, it can be written as
\[
w_h^k=(\Lambda_h^k)^{-1}\sum_{m=1}^M\sum_{t=1}^{\tau_k-1}\phi(x_h^{m,t},a_h^{m,t})\bigl[r_{m,h}(x_h^{m,t},a_h^{m,t})+V_{h+1}^{k}(x_{h+1}^{m,t})\bigr].
\]

\section{Auxiliary Lemmas}\label{app:auxiliary_lemmas}
\begin{lemma}[Lemma 12 in \cite{abbasi2011improved}]\label{lem:norm}
Let $V \in \mathbb{R}^{d \times d}$ be a positive definite matrix and 
$x \in \mathbb{R}^d$. Then,
\[
\min\left\{1,\; \|x\|_{V^{-1}}^2 \right\}
\;\le\;
2 \log \frac{\det\!\bigl(V + x x^\top\bigr)}{\det(V)}.
\]
\end{lemma}

\begin{lemma}[Proposition 2.3 and Lemma B.1 in \cite{jin2020provably}]\label{lem:wpi_bound}
Assume the linear MDP parameters satisfy $\|\theta_h\|_2\le \sqrt d$ and $\|\mu_h(\mathcal X)\|_2\le \sqrt d$ for all $h\in[H]$.
For any policy $\pi$, define
\[
w_h^\pi\coloneqq \theta_h+\int_{\mathcal X}V_{h+1}^\pi(x')\,\mu_h(\mathrm d x').
\]
Then $Q_h^{\pi}(x,a) = \langle \phi(x,a),w_h^\pi\rangle$ for any $(x,a,h) \in \mathcal X\times \mathcal A\times [H]$ and  $\|w_h^\pi\|_2\le 2H\sqrt d$ for any $h \in [H]$.
\end{lemma}

\begin{lemma}[Norm control for the step-wise regression vector]\label{lem:1}
For every step $h\in[H]$ and round $k\in[K]$, it holds that
\[
\|w_h^k\|_2\le 2H\sqrt{\frac{dM\tau_k}{\lambda}}.
\]
\end{lemma}
\begin{proof}[Proof of \texorpdfstring{\Cref{lem:1}}{Lemma B.3}]
For any vector $v\in\mathbb R^d$, since
\[
w_h^k=(\Lambda_h^k)^{-1}\sum_{m=1}^M\sum_{j=1}^{\tau_k-1}\phi(x_h^{m,j},a_h^{m,j})\bigl[r_h(x_h^{m,j},a_h^{m,j})+V_{h+1}^k(x_{h+1}^{m,j})\bigr],
\]
we have
\[
\big|\langle v,w_h^k\rangle\big|
=
\Bigg|\Bigg\langle
v,
\sum_{m=1}^M\sum_{j=1}^{\tau_k-1}\phi(x_h^{m,j},a_h^{m,j})\bigl[r_h(x_h^{m,j},a_h^{m,j})+V_{h+1}^k(x_{h+1}^{m,j})\bigr]
\Bigg\rangle_{(\Lambda_h^k)^{-1}}\Bigg|.
\]
Since rewards are bounded in $[0,1]$ and $0\le V_{h+1}^k\le H$, each regression target satisfies
\[
\bigl|r_h(x_h^{m,j},a_h^{m,j})+V_{h+1}^k(x_{h+1}^{m,j})\bigr|\le 2H.
\]
Applying this bound gives
\begin{align}
\big|\langle v,w_h^k\rangle\big|
&\le 2H\sum_{m=1}^M\sum_{j=1}^{\tau_k-1}\big|\langle v,\phi(x_h^{m,j},a_h^{m,j})\rangle_{(\Lambda_h^k)^{-1}}\big|
\notag\\
&\le 2H\sum_{m=1}^M\sum_{j=1}^{\tau_k-1}\sqrt{\langle v,v\rangle_{(\Lambda_h^k)^{-1}}\,\langle \phi(x_h^{m,j},a_h^{m,j}),\phi(x_h^{m,j},a_h^{m,j})\rangle_{(\Lambda_h^k)^{-1}}}
\notag\\
&\le 2H\,\|v\|_2\,\lambda^{-1/2}\sqrt{M\tau_k\sum_{m=1}^M\sum_{j=1}^{\tau_k-1}\phi(x_h^{m,j},a_h^{m,j})^\top(\Lambda_h^k)^{-1}\phi(x_h^{m,j},a_h^{m,j})}.
\label{eq:whk-bound}
\end{align}
The second inequality follows from Cauchy--Schwarz inequality. The last inequality follows from $\langle v,v\rangle_{(\Lambda_h^k)^{-1}}\leq \|v\|^2_2/\lambda$ since $\Lambda_h^k\succeq \lambda I_d$. Since $\Lambda_h^k$ is a positive-definite matrix, we further note that,
\begin{align}
&\sum_{m=1}^M\sum_{j=1}^{\tau_k-1}\phi(x_h^{m,j},a_h^{m,j})^\top(\Lambda_h^k)^{-1}\phi(x_h^{m,j},a_h^{m,j})
\notag\\&
=\mathrm{tr}\Bigl((\Lambda_h^k)^{-1}\sum_{m=1}^M\sum_{j=1}^{\tau_k-1}\phi(x_h^{m,j},a_h^{m,j})\phi(x_h^{m,j},a_h^{m,j})^\top\Bigr)
\notag\\
&=\mathrm{tr}\bigl((\Lambda_h^k)^{-1}(\Lambda_h^k-\lambda I_d)\bigr)
\notag\\
&=\mathrm{tr}\bigl(I_d-\lambda(\Lambda_h^k)^{-1}\bigr)
\le d.
\label{eq:trace-bound}
\end{align}
Combining \eqref{eq:whk-bound} and \eqref{eq:trace-bound} yields
\[
\big|\langle v,w_h^k\rangle\big|\le 2H\,\|v\|_2\sqrt{\frac{dM\tau_k}{\lambda}}.
\]
Taking the supremum over all $\|v\|_2=1$ concludes the proof.
\end{proof}

\begin{lemma}[Self-Normalized Concentration Inequality, Theorem 1 in {\cite{abbasi2011improved}}]\label{lem:self_normal}
Let $\{\eta_t\}_{t \ge 1}$ be a real-valued martingale difference sequence
adapted to a filtration $\{\mathcal{F}_t\}$, such that $\eta_t$ is conditionally
$R$-sub-Gaussian, i.e.,
\[
\mathbb{E}\!\left[\exp(\lambda \eta_t)\mid \mathcal{F}_{t-1}\right]
\le \exp\!\left(\frac{\lambda^2 R^2}{2}\right)
\quad \text{for all } \lambda \in \mathbb{R}.
\]

Let $\{x_t\}_{t \ge 1}$ be an $\mathbb{R}^d$-valued predictable process
(i.e., $x_t$ is $\mathcal{F}_{t-1}$-measurable), and define
\[
V_t = \lambda I + \sum_{s=1}^t x_s x_s^\top,
\quad
S_t = \sum_{s=1}^t x_s \eta_s,
\]
for some $\lambda > 0$.

Then, for any $\delta \in (0,1)$, with probability at least $1-\delta$,
for all $t \ge 0$ simultaneously,
\[
\|S_t\|_{V_t^{-1}}^2
\;\le\;
2 R^2 \log\left(
\frac{\det(V_t)^{1/2}}{\det(\lambda I)^{1/2}\,\delta}
\right).
\]
\end{lemma}

\begin{lemma}[Self-normalized concentration for a fixed target function]\label{lem:self-normalized-uniform}
Fix a step $h\in[H]$ and a function $V:\mathcal X\to[0,H]$.
Then, for any confidence level $p\in(0,1)$, with probability at least $1-p$, simultaneously for all $t\in\{1,\ldots,T+1\}$,
\[
\begin{aligned}
&
\Bigg\|
\sum_{m=1}^M\sum_{j=1}^{t-1}
\phi(x_h^{m,j},a_h^{m,j})
\Bigl[
V(x_{h+1}^{m,j})
-\mathbb P_hV(x_h^{m,j},a_h^{m,j})
\Bigr]
\Bigg\|_{(\check\Lambda_h^t)^{-1}}
\\
&\qquad\le
2H\sqrt{
\frac{d}{2}\log\Bigl(1+\frac{MT}{d\lambda}\Bigr)
+
2\log\Bigl(\frac{1}{p}\Bigr)
}.
\end{aligned}
\]
\end{lemma}
\begin{proof}[Proof of \texorpdfstring{\Cref{lem:self-normalized-uniform}}{Lemma B.5}]
Recall that \[\mathcal G_{m,t,h}
= \sigma\!\Big(
\{(x_\ell^{n,j},a_\ell^{n,j})\}_{(n,j)<(m,t),\,\ell\in[H]}
\cup
\{(x_\ell^{m,t},a_\ell^{m,t})\}_{\ell\le h}
\Big),\]
\(\eta_{m,j}\coloneqq V(x_{h+1}^{m,j})-\mathbb{P}_hV(x_h^{m,j},a_h^{m,j})\) is a bounded martingale difference with respect to $\mathcal G_{m,t,h}$.

Applying Lemma~\ref{lem:self_normal} gives that with probability at least $1-p$, simultaneously for all $t$,
\[
\Bigg\|\sum_{m=1}^M\sum_{j=1}^{t-1}\phi(x_h^{m,j},a_h^{m,j})\eta_{m,j}\Bigg\|_{(\check{\Lambda}_h^t)^{-1}}^2
\le 4H^2\log\Bigg(\frac{\det(\check{\Lambda}_h^{T+1})^{1/2}}{p^2\det(\check{\Lambda}_h^{1})^{1/2}}\Bigg).
\]
Finally, we upper bound the log-determinant term by using $\det(\lambda I+\sum_{i=1}^{MT}\phi_i\phi_i^\top)\le \lambda^d\bigl(1+MT/(\lambda d)\bigr)^d$.
This yields
\[
\log\Bigg(\frac{\det(\check{\Lambda}_h^{T+1})}{p^2\det(\check{\Lambda}_h^{1})}\Bigg)
\le \frac{d}{2}\log\Bigl(1+\frac{MT}{d\lambda}\Bigr)+2\log\Bigl(\frac{1}{p}\Bigr),
\]
which concludes the proof.
\end{proof}

\begin{lemma}[Elliptical potential lemma, Lemma 11 in {\cite{abbasi2011improved}}]\label{lem:elliptical_sum}
Let $\{x_t\}_{t=1}^T \subset \mathbb{R}^d$ with $\|x_t\|_2 \le 1$, and define
\[
V_t = \lambda I + \sum_{s=1}^{t} x_s x_s^\top,
\]
for some $\lambda > 0$. Then,
\[
\sum_{t=1}^T \min\bigl\{1, \|x_t\|_{V_{t-1}^{-1}}^2\bigr\}
\;\le\;
2 \log \frac{\det(V_T)}{\det(\lambda I)}
\;\le\;
2d \log\Bigl(1 + \frac{T}{d\lambda}\Bigr).
\]
\end{lemma}

\begin{lemma}[Covering-number bound for optimistic value functions, Lemma D.6 in \cite{jin2020provably}]\label{lem:covering_number_in_jin}
Let $\mathcal V$ denote a class of functions of the form
\[
V(\cdot)=\min\Bigl\{\max_{a\in\mathcal A}\bigl[w^\top\phi(\cdot,a)+\beta\sqrt{\phi(\cdot,a)^\top\Lambda^{-1}\phi(\cdot,a)}\bigr],H\Bigr\},
\]
where the parameters satisfy $\|w\|_2\le L$, $\beta\in[0,B]$, and $\lambda_{\min}(\Lambda)\ge \lambda$.
Assume $\|\phi(x,a)\|_2\le 1$ for all $(x,a)$.
Let $\mathcal N_\epsilon$ be the $\epsilon$-covering number of $\mathcal V$ under
$\mathrm{dist}(V,V')=\sup_x|V(x)-V'(x)|$.
Then
\[
\log \mathcal N_\epsilon
\le d\log\Bigl(1+\frac{4L}{\epsilon}\Bigr)
+d^2\log\Bigl(1+\frac{8\sqrt d\,B^2}{\lambda\epsilon^2}\Bigr).
\]
\end{lemma}

\begin{lemma}[Covering-number bound for optimistic value functions]\label{lem:covering_number}
Let $\mathcal V$ denote a class of functions of the form
\[
V(\cdot)=\Pi_{[0,H]}\Bigl(\max_{a\in\mathcal A}\bigl[w^\top\phi(\cdot,a)+\beta\sqrt{\phi(\cdot,a)^\top\Lambda^{-1}\phi(\cdot,a)}\bigr]\Bigr),
\]
where the parameters satisfy $\|w\|_2\le L$, $\beta\in[0,B]$, and $\lambda_{\min}(\Lambda)\ge \lambda$.
Assume $\|\phi(x,a)\|_2\le 1$ for all $(x,a)$.
Let $\mathcal N_\epsilon$ be the $\epsilon$-covering number of $\mathcal V$ under
$\mathrm{dist}(V,V')=\sup_x|V(x)-V'(x)|$.
Then
\[
\log \mathcal N_\epsilon
\le d\log\Bigl(1+\frac{4L}{\epsilon}\Bigr)
+d^2\log\Bigl(1+\frac{8\sqrt d\,B^2}{\lambda\epsilon^2}\Bigr).
\]
\end{lemma}
\begin{proof}
Let $\bar{\mathcal V}$ denote the value-function class before the additional lower clipping, namely
\[
\bar V(\cdot)
=
\min\left\{
\max_{a\in\mathcal A}
\left(
w^\top \phi(\cdot,a)
+
\beta\sqrt{\phi(\cdot,a)^\top\Lambda^{-1}\phi(\cdot,a)}
\right),
H
\right\}.
\]
After adding the projection onto $[0,H]$, the new value-function class can be written as
\[
\mathcal V
=
\left\{
V(\cdot)=[\bar V(\cdot)]_+:\bar V\in\bar{\mathcal V}
\right\},
\]
where $[z]_+=\max\{z,0\}$. Since the map $z\mapsto [z]_+$ is $1$-Lipschitz, for any
$\bar V,\bar V'\in\bar{\mathcal V}$,
\[
\|[\bar V]_+-[\bar V']_+\|_\infty
\le
\|\bar V-\bar V'\|_\infty .
\]
Therefore, if $\{\bar V_i\}_{i=1}^{N_\epsilon}$ is an $\epsilon$-net of $\bar{\mathcal V}$ under the sup norm, then
$\{[\bar V_i]_+\}_{i=1}^{N_\epsilon}$ is an $\epsilon$-net of $\mathcal V$. Hence
\[
N_\epsilon(\mathcal V)
\le
N_\epsilon(\bar{\mathcal V}).
\]
Applying \Cref{lem:covering_number_in_jin} gives the same bound for the projected class $\mathcal V$.
\end{proof}
\begin{lemma}[Uniform concentration for the data-dependent target $V_{h+1}^k$]\label{lem:Vk_self_normalized}
Fix any confidence level $p\in(0,1)$ and $\lambda =1 $, and set $\beta = c_\beta Hd\sqrt{\log(1+dMHT/p)}$, where $c_\beta$ is a sufficiently large positive constant.
Define event $\mathcal C$ that simultaneously for all rounds $k\in[K]$ and steps $h\in[H]$,
\[
\Bigg\|\sum_{m=1}^M\sum_{j=1}^{\tau_k-1}\phi(x_h^{m,j},a_h^{m,j})\Bigl[V_{h+1}^{k}(x_{h+1}^{m,j})-\mathbb P_hV_{h+1}^{k}(x_h^{m,j},a_h^{m,j})\Bigr]\Bigg\|_{(\check\Lambda_h^{\tau_k})^{-1}}^2
\le CH^2d^2\chi.
\]
where $\chi = \log(1+(c_\beta+1)dMHT/p)$.
Then, $\mathbb P(\mathcal C) \geq 1-p$.
\end{lemma}
\begin{proof}[Proof of \texorpdfstring{\Cref{lem:Vk_self_normalized}}{Lemma B.8}]
The challenge is that $V_{h+1}^k$ is data dependent.
We therefore discretize the relevant function class and then transfer the bound from the net to the whole class.

By construction, every $V_{h+1}^k$ belongs to the class $\mathcal V$ in Lemma~\ref{lem:covering_number}.
Moreover, Lemma~\ref{lem:1} implies that the regression parameter can be taken in a ball of radius
$L=2H\sqrt{dMT/\lambda}$, because $\tau_k\le T$ for all $k$.

Let $\{V^{(1)},\ldots,V^{(\mathcal N_\epsilon)}\}$ be an $\epsilon$-net of $\mathcal V$ under the sup norm.
Applying Lemma~\ref{lem:self-normalized-uniform} to each $V^{(i)}$ and taking a union bound over $i$ and $h$ shows that, with probability at least $1-p$, for every $h\in[H]$, every $k\in[K]$, and every net point $V^{(i)}$,
\[
\|S_{h,k}^{(i)}\|_{(\check\Lambda_h^{\tau_k})^{-1}}^2
\le
2H^2
d\log\Bigl(1+\frac{MT}{d\lambda}\Bigr)
+
8H^2\log\Bigl(\frac{H\mathcal N_\epsilon}{p}\Bigr)
.
\]
Here, 
\[
S_{h,k}^{(i)}
:=
\sum_{m=1}^M\sum_{j=1}^{\tau_k-1}
\phi(x_h^{m,j},a_h^{m,j})
\Bigl[
V^{(i)}(x_{h+1}^{m,j})
-\mathbb P_hV^{(i)}(x_h^{m,j},a_h^{m,j})
\Bigr].
\]
Fix any $V\in\mathcal V$ and choose $V^{(i)}$ in the net such that $\|V-V^{(i)}\|_\infty\le \epsilon$.
For every sample,
\[
\Big|\bigl(V-V^{(i)}\bigr)(x_{h+1}^{m,j})-\mathbb P_h\bigl(V-V^{(i)}\bigr)(x_h^{m,j},a_h^{m,j})\Big|\le 2\epsilon.
\]
Therefore,
\[
\Bigg\|\sum_{m=1}^M\sum_{j=1}^{\tau_k-1}\phi(x_h^{m,j},a_h^{m,j})\Bigl[(V-V^{(i)})(x_{h+1}^{m,j})-\mathbb P_h(V-V^{(i)})(x_h^{m,j},a_h^{m,j})\Bigr]\Bigg\|_{(\check\Lambda_h^{\tau_k})^{-1}}^2
\le \frac{4M^2T^2\epsilon^2}{\lambda},
\]
where we used $\check\Lambda_h^k\succeq \lambda I_d$ and $\|\phi(x_h^{m,j},a_h^{m,j})\|_2\le 1$.
Combining the net bound with this approximation error, we have
\[
\Bigg\|\sum_{m=1}^M\sum_{j=1}^{\tau_k-1}\phi(x_h^{m,j},a_h^{m,j})\Bigl[V_{h+1}^{k}(x_{h+1}^{m,j})-\mathbb P_hV_{h+1}^{k}(x_h^{m,j},a_h^{m,j})\Bigr]\Bigg\|_{(\check\Lambda_h^{\tau_k})^{-1}}^2
\le \Gamma_p.
\]
where
\[
\Gamma_p
\coloneqq
8H^2 \Bigl(\frac{d}{2}\log\Bigl(1+\frac{MT}{d\lambda}\Bigr)+2\log\Bigl(\frac{H}{p}\Bigr)+2\log \mathcal N_\epsilon \Bigr)
+\frac{8M^2T^2\epsilon^2}{\lambda},
\]
and $\mathcal N_\epsilon$ is the covering number of the value-function class from Lemma~\ref{lem:covering_number}.
Let $\lambda = 1$, $\epsilon = Hd/MT$, $\beta = c_\beta Hd\sqrt{\log (1+dMHT/p)}$. The $\Gamma_p$ becomes
\begin{equation}
\begin{aligned}
\Gamma_p
= {8}& H^2 \Biggl(
\frac{d}{2}\log\Bigl(1+\frac{MT}{d}\Bigr)
+2\log\Bigl(\frac{H}{p}\Bigr) 
+2d\log\Bigl(1+\frac{8(MT)^{3/2}}{\sqrt d}\Bigr) \\
&\qquad
+2d^2\log\Bigl(1+8c_\beta\sqrt{d}M^2T^2 \log(1+dMHT/p)\Bigr)
\Biggr)
+8H^2d^2 .
\end{aligned}
\label{eq:Gamma-p}
\end{equation}
Thus, we have:
\[\Gamma_p \leq CH^2d^2\chi,\] where $\chi = \log(1+(c_\beta+1)dMHT/p)$ and $C$ is an absolute constant independent of $c_\beta$.
\end{proof}

\section{Proof of \texorpdfstring{\Cref{theo_1}}{Theorem 1}} \label{app:regret_analysis}
\subsection{Supporting Lemmas}
\label{susbsec_support}
For the trajectory generated by agent $m$ in episode $t$, define
\begin{align}
\delta_h^{m,t}
&\coloneqq V_h^{k(t)}(x_h^{m,t})-V_h^{\pi^{k(t)}}(x_h^{m,t}),\nonumber\\
\mathcal G_{m,t,h}
&\coloneqq \sigma\!\Big(
\{(x_\ell^{n,j},a_\ell^{n,j})\}_{(n,j)<(m,t),\,\ell\in[H]}
\cup
\{(x_\ell^{m,t},a_\ell^{m,t})\}_{\ell\le h}
\Big),\nonumber\\
\zeta_{h+1}^{m,t}
&\coloneqq \mathbb E\big[\delta_{h+1}^{m,t}\mid \mathcal G_{m,t,h}\big]-\delta_{h+1}^{m,t}.\nonumber
\end{align}
Equivalently,
\[
\zeta_{h+1}^{m,t}=\mathbb P_h\big(V_{h+1}^{k(t)}-V_{h+1}^{\pi^{k(t)}}\big)(x_h^{m,t},a_h^{m,t})-\big(V_{h+1}^{k(t)}-V_{h+1}^{\pi^{k(t)}}\big)(x_{h+1}^{m,t}).
\]

\begin{lemma}[Uniform confidence decomposition for the regression error]\label{lem:2}
Set $\lambda=  1$. There exists an absolute constant $c_\beta>0$ such that for any confidence level $p\in(0,1)$, if we take $\beta = c_\beta Hd\sqrt{\log (1+dMHT/p)}$, then conditioned on the event $\mathcal C$ defined in~\Cref{lem:Vk_self_normalized}, the following statement holds simultaneously for all rounds $k\in[K]$, steps $h\in[H]$, and state-action pairs $(x,a) \in \mathcal{X} \times \mathcal{A}$,
\[
\bigl|\phi(x,a)^\top w_h^k-Q_h^\pi(x,a)-\mathbb P_h\big(V_{h+1}^{k}-V_{h+1}^\pi\big)(x,a)\bigr|
\le
\beta\,\|\phi(x,a)\|_{(\Lambda_h^k)^{-1}},
\]
\end{lemma}

\begin{lemma}\label{lem:3}
 Conditioned on the event $\mathcal C$ defined in~\Cref{lem:Vk_self_normalized}, for every round $k\in[K]$, step $h\in[H]$ and state-action pair $(x,a) \in \mathcal{X} \times \mathcal A$, the optimistic estimates upper bound the corresponding value functions:
\[
Q_h^k(x,a)\ge Q_h^*(x,a),\quad  V_h^k(x)\ge V_h^*(x).
\]
\end{lemma}

\begin{lemma}\label{lem:4}
For every agent $m\in[M]$, episode $t\in[T]$, and step $h\in[H]$,
\[
\delta_h^{m,t}
\le
\delta_{h+1}^{m,t}+\zeta_{h+1}^{m,t}+2\beta\,\|\phi(x_h^{m,t},a_h^{m,t})\|_{(\Lambda_h^{k(t)})^{-1}}.
\]
\end{lemma}

\begin{lemma}[Concentration of the martingale correction term]\label{lem:5}
For any confidence level $p\in(0,1)$, define the event $\mathcal E$ such that,
\[
\sum_{m=1}^M\sum_{t=1}^T\sum_{h=1}^H\zeta_{h+1}^{m,t}
\le
\sqrt{8H^3MT\log\Bigl(\frac{2}{p}\Bigr)},
\]
Then $\mathbb P(\mathcal E)\geq 1-p$.
\end{lemma}

\begin{lemma}[Elliptical potential bound]\label{lem:6}
For any $\lambda\ge 1$ and $\gamma\ge 1$,
\[
\begin{aligned}
&\sum_{h=1}^H\sum_{k=1}^K\sum_{t=\tau_k}^{\tau_{k+1}-1}\sum_{m=1}^M
\|\phi(x_h^{m,t},a_h^{m,t})\|_{(\Lambda_h^k)^{-1}}
\le \\
&H\Bigg(
\sqrt{6dMT\log\Bigl(1+\frac{MT}{d\lambda}\Bigr)}
+\frac{Md}{\log 3}\sqrt{\gamma\log\gamma}\,\log\Bigl(1+\frac{MT}{d\lambda}\Bigr)+\frac{Md}{\sqrt\lambda\,\log 3}\,\log\Bigl(1+\frac{MT}{d\lambda}\Bigr)
\Bigg).
\end{aligned}
\]
\end{lemma}

\subsection{Proof of Theorem~\ref{theo_1}}
\begin{proof}[Proof of Theorem~\ref{theo_1}]
The event $\mathcal C\cap\mathcal E$ holds with probability at least $1-2p$. We now prove \Cref{theo_1} under the event $\mathcal{C}\cap\mathcal{E}$, under which all supporting lemmas in \Cref{susbsec_support} hold.

We first expand the regret and reduce it to bounding $\sum_{m,t}\delta_1^{m,t}$. For each episode $t$ and each agent $m$, the executed policy is exactly $\pi^{k(t)}$.
Lemma~\ref{lem:3} implies $V_1^*(x)\le V_1^{k(t)}(x)$ for every state $x$, hence
\begin{align*}
\mathrm{Regret}(T)
&=\sum_{m=1}^M\sum_{t=1}^T\Bigl(V_1^*(x_1^{m,t})-V_1^{\pi^{k(t)}}(x_1^{m,t})\Bigr)
\\
&\le \sum_{m=1}^M\sum_{t=1}^T\Bigl(V_1^{k(t)}(x_1^{m,t})-V_1^{\pi^{k(t)}}(x_1^{m,t})\Bigr)
\\
&=\sum_{m=1}^M\sum_{t=1}^T\delta_1^{m,t}.
\end{align*}

Summing Lemma~\ref{lem:4} over $h=1,\ldots,H$ and using $\delta_{H+1}^{m,t}=0$ gives
\[
\delta_1^{m,t}
\le
\sum_{h=1}^H\zeta_{h+1}^{m,t}
+2\beta\sum_{h=1}^H\|\phi(x_h^{m,t},a_h^{m,t})\|_{(\Lambda_h^{k(t)})^{-1}}.
\]
Summing over all $(m,t)$ and regrouping by rounds yields
\[
\sum_{m=1}^M\sum_{t=1}^T\delta_1^{m,t}
\le
\underbrace{\sum_{m=1}^M\sum_{t=1}^T\sum_{h=1}^H\zeta_{h+1}^{m,t}}_{\text{martingale term}}
+2\beta\underbrace{\sum_{k=1}^K\sum_{t=\tau_k}^{\tau_{k+1}-1}\sum_{m=1}^M\sum_{h=1}^H\|\phi(x_h^{m,t},a_h^{m,t})\|_{(\Lambda_h^k)^{-1}}}_{\text{elliptical-potential term}}.
\]

On the event $\mathcal C\cap\mathcal E$, we have
\begin{align*}
\mathrm{Regret}(T)
&\leq\;
\sqrt{8H^3MT\log\Bigl(\frac{2}{p}\Bigr)} + 2H\beta \Bigg(
\sqrt{6dMT\log\Bigl(1+\frac{MT}{d\lambda}\Bigr)} \notag\\
&\quad + \frac{Md}{\log 3}\sqrt{\gamma\log\gamma}\,
\log\Bigl(1+\frac{MT}{d\lambda}\Bigr)  + \frac{Md}{\sqrt{\lambda}\,\log 3}\,
\log\Bigl(1+\frac{MT}{d\lambda}\Bigr)
\Bigg).
\end{align*}
where $\beta = c_\beta Hd\sqrt{\log (1+dMHT/p)}$. By setting $\lambda = 1$ and plugging in $\beta$, the above bound directly recovers the regret bound stated in Theorem~\ref{theo_1}: 
\begin{align}
\mathrm{Regret}(T)
&\leq\; (2c_\beta\sqrt{6}+4\sqrt{2})\sqrt{Md^3H^4T\log \Bigl(1+\frac{dMHT}{p}\Bigr)\log\Bigl(1+\frac{MT}{d}\Bigr)}   \notag\\
&\quad + \frac{2c_\beta Md^2H^2}{\log 3}(\sqrt{\gamma\log\gamma}+1)\,
\log\Bigl(1+\frac{MT}{d}\Bigr). \label{finalregret}
\end{align}

\end{proof}

\subsection{Proof of Supporting Lemmas}

\subsubsection{Proof of Lemma \texorpdfstring{\ref{lem:2}}{C.1}}
\begin{proof}[Proof of \Cref{lem:2}]
Fix a round $k$, a step $h$, a policy $\pi$, and a state-action pair $(x,a)$.
We begin from an exact algebraic decomposition of $w_h^k-w_h^\pi$.
By \eqref{eq_Bellman} and \Cref{lem:wpi_bound}, we have
\[
Q_h^\pi(x,a)=r_h(x,a)+\mathbb P_hV_{h+1}^\pi(x,a)=\phi(x,a)^\top w_h^\pi,
\]
we obtain
\begin{align}
&w_h^k-w_h^\pi
=(\Lambda_h^k)^{-1}\sum_{m=1}^M\sum_{j=1}^{\tau_k-1}\phi(x_h^{m,j},a_h^{m,j})
\bigl[r_h(x_h^{m,j},a_h^{m,j})+V_{h+1}^k(x_{h+1}^{m,j})\bigr]-w_h^\pi
\notag\\
&=(\Lambda_h^k)^{-1}\sum_{m=1}^M\sum_{j=1}^{\tau_k-1}\phi(x_h^{m,j},a_h^{m,j})
\Bigl[V_{h+1}^k(x_{h+1}^{m,j})-\mathbb P_hV_{h+1}^\pi(x_h^{m,j},a_h^{m,j})\Bigr]
-\lambda(\Lambda_h^k)^{-1}w_h^\pi
\notag\\
&=(\Lambda_h^k)^{-1}\sum_{m=1}^M\sum_{j=1}^{\tau_k-1}\phi(x_h^{m,j},a_h^{m,j})
\Bigl[V_{h+1}^k(x_{h+1}^{m,j})-\mathbb P_hV_{h+1}^k(x_h^{m,j},a_h^{m,j})\Bigr]
\notag\\
&\quad+(\Lambda_h^k)^{-1}\sum_{m=1}^M\sum_{j=1}^{\tau_k-1}\phi(x_h^{m,j},a_h^{m,j})\,
\mathbb P_h\bigl(V_{h+1}^k-V_{h+1}^\pi\bigr)(x_h^{m,j},a_h^{m,j})
-\lambda(\Lambda_h^k)^{-1}w_h^\pi.
\label{eq:wh-diff-decomp}
\end{align}
Now define
\[
\nu_h^{k,\pi}\coloneqq \int_{\mathcal X}\bigl(V_{h+1}^k(s')-V_{h+1}^\pi(s')\bigr)\,\mu_h(\mathrm d s').
\]
By linearity of the transition probability,
\begin{equation}
\label{pv}
    \mathbb P_h\bigl(V_{h+1}^k-V_{h+1}^\pi\bigr)(x,a)=\phi(x,a)^\top \nu_h^{k,\pi}.
\end{equation}
Therefore,
\begin{align*}
(\Lambda_h^k)^{-1}\sum_{m=1}^M\sum_{j=1}^{\tau_k-1}\phi(x_h^{m,j},a_h^{m,j})\,\mathbb P_h\bigl(V_{h+1}^k-V_{h+1}^\pi\bigr)(x_h^{m,j},a_h^{m,j})
&=(\Lambda_h^k)^{-1}(\Lambda_h^k-\lambda I_d)\nu_h^{k,\pi}
\\
&=\nu_h^{k,\pi}-\lambda(\Lambda_h^k)^{-1}\nu_h^{k,\pi}.
\end{align*}
Combined with the \eqref{eq:wh-diff-decomp} and \eqref{pv}, we arrive at
\begin{equation}\label{eq:q-decomp}
\begin{aligned}
&\phi(x,a)^\top w_h^k-Q_h^\pi(x,a)-\mathbb P_h\bigl(V_{h+1}^k-V_{h+1}^\pi\bigr)(x,a)
\\
&=\langle
\phi(x,a),
w_h^k-w_h^\pi-\nu_h^{k,\pi}\rangle \\
&=
\Bigg\langle
\phi(x,a),
\sum_{m=1}^M\sum_{j=1}^{\tau_k-1}\phi(x_h^{m,j},a_h^{m,j})
\Bigl[V_{h+1}^k(x_{h+1}^{m,j})-\mathbb P_hV_{h+1}^k(x_h^{m,j},a_h^{m,j})\Bigr]
\Bigg\rangle_{(\Lambda_h^k)^{-1}}
\\
&\quad-\lambda\,\big\langle \phi(x,a),w_h^\pi+\nu_h^{k,\pi}\big\rangle_{(\Lambda_h^k)^{-1}}.
\end{aligned}
\end{equation}
For the first term of \eqref{eq:q-decomp}, Cauchy-Schwarz and conditioned on the event $\mathcal C$ give
\begin{align}
\label{eq:concentrate}
&\Bigg|
\Bigg\langle
\phi(x,a),\,
\sum_{m=1}^M\sum_{j=1}^{\tau_k-1}
\phi(x_h^{m,j},a_h^{m,j})
\Bigl[
V_{h+1}^k(x_{h+1}^{m,j})-\mathbb P_hV_{h+1}^k(x_h^{m,j},a_h^{m,j})
\Bigr]
\Bigg\rangle_{(\Lambda_h^k)^{-1}}
\Bigg|
\notag\\
&\le
\|\phi(x,a)\|_{(\Lambda_h^k)^{-1}}
\Bigg\|
\sum_{m=1}^M\sum_{j=1}^{\tau_k-1}
\phi(x_h^{m,j},a_h^{m,j})
\Bigl[
V_{h+1}^k(x_{h+1}^{m,j})
-\mathbb P_hV_{h+1}^k(x_h^{m,j},a_h^{m,j})
\Bigr]
\Bigg\|_{(\Lambda_h^k)^{-1}}
\notag\\
&\le
CHd\sqrt{\chi}\,\|\phi(x,a)\|_{(\Lambda_h^k)^{-1}}.
\end{align}

For the second term of \eqref{eq:q-decomp}, note that
\[
w_h^\pi+\nu_h^{k,\pi}=\theta_h+\int_{\mathcal X}V_{h+1}^k(s')\,\mu_h(\mathrm d s').
\]
Similar with the proof in \Cref{lem:wpi_bound}, we have
$\|w_h^\pi+\nu_h^{k,\pi}\|_2\le 2H\sqrt d$.
Using again $\Lambda_h^k\succeq  I_d$,
\begin{align}
\lambda\,\big|\langle \phi(x,a),w_h^\pi+\nu_h^{k,\pi}\rangle_{(\Lambda_h^k)^{-1}}\big|
&\le \|\phi(x,a)\|_{(\Lambda_h^k)^{-1}}\,\|w_h^\pi+\nu_h^{k,\pi}\|_{(\Lambda_h^k)^{-1}} \le 2H\sqrt{d}\,\|\phi(x,a)\|_{(\Lambda_h^k)^{-1}}.
\label{eq:residule}
\end{align}
Combining the two bounds \eqref{eq:concentrate} and \eqref{eq:residule}, we have
\begin{equation}
\big|\phi(x,a)^\top w_h^k - Q_h^\pi(x,a)
- \mathbb P_h\bigl(V_{h+1}^k - V_{h+1}^\pi\bigr)(x,a)\big|
\le (2H\sqrt{d}+CHd\sqrt{\chi})\,
\|\phi(x,a)\|_{(\Lambda_h^k)^{-1}}.
\label{eq:q-diff-bound}
\end{equation}
There exists an absolute constant \(c_\beta > 0\) such that for all \(d,H,M,T \ge 1\),
\[
2H\sqrt{d} + C H d \sqrt{\chi}
\;\le\;
c_\beta H d \sqrt{\log\bigl(1 + dMHT/p\bigr)}.
\]
Indeed, it suffices to verify that there exists \(c_\beta > 0\) such that for all \(x \ge 1\),
\[
2 + C \sqrt{\log\bigl(1 + (c_\beta+1)x\bigr)}
\;\le\;
c_\beta \sqrt{\log(1 + x)}.
\]
This condition can be satisfied by choosing \(c_\beta = 8(C+1)^2\).
\end{proof}

\subsubsection{Proof of Lemma \texorpdfstring{\ref{lem:3}}{C.2}}
\begin{proof}[Proof of \Cref{lem:3}]
We prove optimism by backward induction on $h$.

\paragraph{Base case \texorpdfstring{$h=H$}{h = H}.}
At the last step, $V_{H+1}^k\equiv V_{H+1}^*\equiv 0$.
Applying Lemma~\ref{lem:2} with $\pi=\pi^*$ gives
\[
\phi(x,a)^\top w_H^k\ge Q_H^*(x,a)-\beta\,\|\phi(x,a)\|_{(\Lambda_H^k)^{-1}}.
\]
By the definition of the optimistic estimator, for any $(x,a,k) \in \mathcal{X} \times \mathcal A \times [K]$, it holds that
\[
Q_H^k(x,a)
=\Pi_{[0,H]}\Bigl(\phi(x,a)^\top w_H^k+\beta\,\|\phi(x,a)\|_{(\Lambda_H^k)^{-1}}\Bigr)
\ge Q_H^*(x,a).
\]
Therefore $ V_H^k(x) =\max\limits_a Q^k_H(x,a)  \ge \max\limits_a Q^*_H(x,a) = V_H^*(x)$.

\paragraph{Induction step.}
Assume that $Q_{h+1}^k(x,a)\ge Q_{h+1}^*(x,a)$ and $V_{h+1}^k(x)\ge V_{h+1}^*(x)$ for any $(x,a,k) \in \mathcal{X} \times \mathcal A \times [K]$.
Then $\mathbb P_h(V_{h+1}^k-V_{h+1}^*)(x,a)\ge 0$ for any $(x,a,k)$.

Applying Lemma~\ref{lem:2} with $\pi=\pi^*$ yields
\[
\begin{aligned}
\phi(x,a)^\top w_h^k
&\ge Q_h^*(x,a)+\mathbb P_h\bigl(V_{h+1}^k-V_{h+1}^*\bigr)(x,a)-\beta\,\|\phi(x,a)\|_{(\Lambda_h^k)^{-1}}
\\
&\ge Q_h^*(x,a)-\beta\,\|\phi(x,a)\|_{(\Lambda_h^k)^{-1}}.
\end{aligned}
\]
Therefore, for any $(x,a,k) \in \mathcal{X} \times \mathcal A \times [K]$, we have
$$Q_h^k(x,a)=\Pi_{[0,H]}\Bigl(\phi(x,a)^\top w_h^k+\beta\,\|\phi(x,a)\|_{(\Lambda_h^k)^{-1}}\Bigr)\ge Q_h^*(x,a)$$ and thus $$ V_h^k(x) =\max\limits_a Q^k_h(x,a)  \ge \max\limits_a Q^*_h(x,a) = V_h^*(x).$$
This completes the proof for step $h$, and hence the conclusion follows by backward induction.
\end{proof}

\subsubsection{Proof of Lemma \texorpdfstring{\ref{lem:4}}{C.3}}
\begin{proof}[Proof of \Cref{lem:4}]
Let $k=k(t)$ be the round index that the $t$-th episode is in. Fix an agent $m$, an episode $t$. Since $\pi^{k(t)}$ is greedy with respect to $Q_h^k$,
\[
V_h^{k(t)}(x_h^{m,t})=Q_h^{k(t)}(x_h^{m,t},a_h^{m,t}),
\qquad
V_h^{\pi^{k(t)}}(x_h^{m,t})=Q_h^{\pi^{k(t)}}(x_h^{m,t},a_h^{m,t}).
\]
Therefore,
\[
\delta_h^{m,t}=Q_h^{k(t)}(x_h^{m,t},a_h^{m,t})-Q_h^{\pi^{k(t)}}(x_h^{m,t},a_h^{m,t}).
\]
Applying Lemma~\ref{lem:2} with $\pi=\pi^{k(t)}$ and using the definition of $Q_h^{k(t)}$ gives
\[
\delta_h^{m,t}
\le
\mathbb P_h\bigl(V_{h+1}^{k(t)}-V_{h+1}^{\pi^{k(t)}}\bigr)(x_h^{m,t},a_h^{m,t})
+2\beta\,\|\phi(x_h^{m,t},a_h^{m,t})\|_{(\Lambda_h^k)^{-1}}.
\]
The preceding argument applies unless the lower clipping is active, namely when the unprojected optimistic score is negative. In this case $Q_h^k=0$, and thus
\begin{align*}
\delta_h^{m,t}=-Q_h^{\pi^{k(t)}}(x_h^{m,t},a_h^{m,t})&=-r_h(x_h^{m,t},a_h^{m,t})-\mathbb P_hV_{h+1}^{\pi^{k(t)}}(x_h^{m,t},a_h^{m,t})\\
&\le \mathbb P_h\bigl(V_{h+1}^{k(t)}-V_{h+1}^{\pi^{k(t)}}\bigr)(x_h^{m,t},a_h^{m,t}),
\end{align*}
because $r_h\ge0$ and $V_{h+1}^{k(t)}\ge0$. Combining the two cases and noting that
\[
\mathbb P_h\bigl(V_{h+1}^{k(t)}-V_{h+1}^{\pi^{k(t)}}\bigr)(x_h^{m,t},a_h^{m,t})
=\mathbb E\bigl[\delta_{h+1}^{m,t}\mid \mathcal G_{m,t,h}\bigr]
=\delta_{h+1}^{m,t}+\zeta_{h+1}^{m,t},
\]
we finish the proof of the conclusion.
\end{proof}

\subsubsection{Proof of Lemma \texorpdfstring{\ref{lem:5}}{C.4}}
\begin{proof}[Proof of \Cref{lem:5}]
By construction,
$\mathbb E[\zeta_{h+1}^{m,t}\mid \mathcal G_{m,t,h}]=0$,
so the sequence $\{\zeta_{h+1}^{m,t}\}_{m,t,h}$ is a martingale-difference array under the linearized ordering.
Moreover, both $V_{h+1}^{k(t)}$ and $V_{h+1}^{\pi^{k(t)}}$ lie in $[0,H]$, hence
$|\zeta_{h+1}^{m,t}|\le 2H$.
Applying Azuma--Hoeffding to the sum over all $MTH$ terms gives
\[
\mathbb P\Bigg(\Big|\sum_{m=1}^M\sum_{t=1}^T\sum_{h=1}^H\zeta_{h+1}^{m,t}\Big|>u\Bigg)
\le 2\exp\Bigl(-\frac{u^2}{8H^3MT}\Bigr).
\]
Setting $u=\sqrt{8H^3MT\log(2/p)}$ yields the claim.
\end{proof}

\subsubsection{Proof of Lemma \texorpdfstring{\ref{lem:6}}{C.5}}
\label{proofe5}
\begin{proof}[Proof of \Cref{lem:6}]
Fix a step $h\in[H]$.
We bound
\[
\sum_{k=1}^K\sum_{t=\tau_k}^{\tau_{k+1}-1}\sum_{m=1}^M\|\phi(x_h^{m,t},a_h^{m,t})\|_{(\Lambda_h^k)^{-1}},
\]
and then sum the result over $h=1,\ldots,H$.

\paragraph{Step 1: Split rounds by global determinant growth.}
Define
\[
\mathcal K_1(h)\coloneqq \Bigl\{k:\frac{\det(\Lambda_h^{k+1})}{\det(\Lambda_h^k)}\le 3\Bigr\},
\qquad
\mathcal K_2(h)\coloneqq \Bigl\{k:\frac{\det(\Lambda_h^{k+1})}{\det(\Lambda_h^k)}>3\Bigr\}.
\]

\paragraph{Step 2: Analysis for rounds with small determinant growth.}
Fix $k\in\mathcal K_1(h)$ and any $(m,t)$ with $t\in[\tau_k,\tau_{k+1}-1]$.
Since $\Lambda_h^k\preceq \Lambda_h^{k+1}$, by Lemma~\ref{lem:norm}, for any $x \in \mathbb{R}^d$
\[
\|x\|_{(\Lambda_h^k)^{-1}}^2
\le \frac{\det(\Lambda_h^{k+1})}{\det(\Lambda_h^k)}\,\|x\|_{(\Lambda_h^{k+1})^{-1}}^2.
\]
Hence
\begin{equation}
\label{xnorm}
    \|x\|_{(\Lambda_h^k)^{-1}}\le \sqrt 3\,\|x\|_{(\Lambda_h^{k+1})^{-1}}.
\end{equation}
Recall that
\begin{equation}
\begin{aligned}
\tilde\Lambda_h^{m,t}
&=
\lambda I_d
+
\sum_{(n,j)\le (m,t)}
\phi(x_h^{n,j},a_h^{n,j})\phi(x_h^{n,j},a_h^{n,j})^\top
\notag\\
&=
\Lambda_h^k
+
\sum_{\substack{(n,j)\le (m,t)\\ k(j)=k}}
\phi(x_h^{n,j},a_h^{n,j})\phi(x_h^{n,j},a_h^{n,j})^\top .
\end{aligned}
\end{equation}
Here the last summation is over those pairs $(n,j)\le (m,t)$ such that
episode $j$ lies in round $k$. Hence, it holds that
\begin{equation*}
\Lambda_h^{k+1}\succeq \tilde\Lambda_h^{m,t}\succeq \Lambda_h^k .
\end{equation*}
and thus together with \eqref{xnorm}, we have
\[
\|\phi(x_h^{m,t},a_h^{m,t})\|_{(\Lambda_h^k)^{-1}}
\le \sqrt 3\,\|\phi(x_h^{m,t},a_h^{m,t})\|_{(\Lambda_h^{k+1})^{-1}}
\le \sqrt 3\,\|\phi(x_h^{m,t},a_h^{m,t})\|_{(\tilde\Lambda_h^{m,t})^{-1}}.
\]
Summing the inequality above and applying Cauchy--Schwarz inequality yields
\begin{align*}
\sum_{k\in\mathcal K_1(h)}\sum_{t=\tau_k}^{\tau_{k+1}-1}\sum_{m=1}^M\|\phi(x_h^{m,t},a_h^{m,t})\|_{(\Lambda_h^k)^{-1}}
&\le \sqrt 3\sum_{t=1}^T\sum_{m=1}^M\|\phi(x_h^{m,t},a_h^{m,t})\|_{(\tilde\Lambda_h^{m,t})^{-1}}
\\
&\le \sqrt{3MT\sum_{t=1}^T\sum_{m=1}^M\|\phi(x_h^{m,t},a_h^{m,t})\|_{(\tilde\Lambda_h^{m,t})^{-1}}^2}.
\end{align*}
Because $\tilde{\Lambda}^{m,t}_h$ is updated sequentially in $(m,t)$, we can apply Lemma~\ref{lem:elliptical_sum} here and get
\[
\sum_{t=1}^T\sum_{m=1}^M\|\phi(x_h^{m,t},a_h^{m,t})\|_{(\tilde\Lambda_h^{m,t})^{-1}}^2
\le 2\log\frac{\det(\tilde\Lambda_h^{M,T})}{\det(\lambda I_d)}
\le 2d\log\Bigl(1+\frac{MT}{d\lambda}\Bigr).
\]
Therefore,
\begin{equation}
\sum_{k\in\mathcal K_1(h)}\sum_{t=\tau_k}^{\tau_{k+1}-1}\sum_{m=1}^M
\|\phi(x_h^{m,t},a_h^{m,t})\|_{(\Lambda_h^k)^{-1}}
\le
\sqrt{6dMT\log\Bigl(1+\frac{MT}{d\lambda}\Bigr)}.
\label{eq:k1-sum-bound}
\end{equation}

\paragraph{Step 3: Analysis for rounds with large determinant growth.}
Fix $k\in\mathcal K_2(h)$ and suppose first that $\tau_{k+1}-\tau_k\ge 2$.
For one agent $m$, define the \textbf{agent-wise pseudo covariance matrix}
\[
\bar\Lambda_h^{m,t}=\Lambda_h^k+\sum_{j=\tau_k}^{t}\phi(x_h^{m,j},a_h^{m,j})\phi(x_h^{m,j},a_h^{m,j})^\top,
\qquad t\in[\tau_k,\tau_{k+1}-1].
\]
Recall that $$\Lambda_h^{m,k,\mathrm{loc}} = \sum_{t=\tau_k}^{\tau_{k+1}-1}\phi(x_h^{m,t},a_h^{m,t})\phi(x_h^{m,t},a_h^{m,t})^\top$$ and $$\Lambda_h^{k+1}=\Lambda_h^k+\sum_{m=1}^M\Lambda_h^{m,k,\mathrm{loc}},$$ we have
\[
\bar\Lambda_h^{m,t}\preceq \Lambda_h^k+\Lambda_h^{m,k,\mathrm{loc}}\preceq \Lambda_h^{k+1}.
\]
For any $t\in[\tau_k,\tau_{k+1}-2]$, by \eqref{xnorm}, we have
\[
\|\phi(x_h^{m,t},a_h^{m,t})\|_{(\Lambda_h^k)^{-1}}
\le \sqrt{\frac{\det(\bar\Lambda_h^{m,\tau_{k+1}-2})}{\det(\Lambda_h^k)}}\,\|\phi(x_h^{m,t},a_h^{m,t})\|_{(\bar\Lambda_h^{m,t})^{-1}}.
\]
Summing this inequality over $t=\tau_k,\ldots,\tau_{k+1}-2$ and applying Cauchy--Schwarz inequality, we reach
\begin{align}
&\sum_{t=\tau_k}^{\tau_{k+1}-2}
\|\phi(x_h^{m,t},a_h^{m,t})\|_{(\Lambda_h^k)^{-1}}
\notag\\&\le
\Bigg[
(\tau_{k+1}-\tau_k-1)\,
\frac{\det(\bar\Lambda_h^{m,\tau_{k+1}-2})}{\det(\Lambda_h^k)}\cdot
\sum_{t=\tau_k}^{\tau_{k+1}-2}
\|\phi(x_h^{m,t},a_h^{m,t})\|_{(\bar\Lambda_h^{m,t})^{-1}}^2
\Bigg]^{1/2}
\notag\\
&\le
\Bigg[
(\tau_{k+1}-\tau_k-1)\,
\frac{\det(\bar\Lambda_h^{m,\tau_{k+1}-2})}{\det(\Lambda_h^k)}\cdot
\log\frac{\det(\bar\Lambda_h^{m,\tau_{k+1}-2})}{\det(\Lambda_h^k)}
\Bigg]^{1/2}.
\label{eq:det-bound-local}
\end{align}
The last inequality is by  Lemma~\ref{lem:elliptical_sum}. Here, we note that the construction of $\bar\Lambda_h^{m,t}$ guarantees the use of Lemma~\ref{lem:elliptical_sum}, because the sum grows in episode order.
Now use the stopping rule.
At time $t=\tau_{k+1}-2$ the round has not yet stopped, so the determinant trigger has not fired.
Hence
\[
\frac{\det(\bar\Lambda_h^{m,\tau_{k+1}-2})}{\det(\Lambda_h^k)}
\le \frac{\gamma}{\tau_{k+1}-\tau_k-1}.
\]
Combining this bound and \eqref{eq:det-bound-local} yields
\[
\sum_{t=\tau_k}^{\tau_{k+1}-2}\|\phi(x_h^{m,t},a_h^{m,t})\|_{(\Lambda_h^k)^{-1}}
\le \sqrt{\gamma\log\Bigl(\frac{\gamma}{\tau_{k+1}-\tau_k-1}\Bigr)}
\le \sqrt{\gamma\log\gamma}.
\]
Summing over $m$ and then over $k\in\mathcal K_2(h)$ gives
\[
\sum_{k\in\mathcal K_2(h)}\sum_{t=\tau_k}^{\tau_{k+1}-2}\sum_{m=1}^M\|\phi(x_h^{m,t},a_h^{m,t})\|_{(\Lambda_h^k)^{-1}}
\le M\,|\mathcal K_2(h)|\,\sqrt{\gamma\log\gamma}.
\]
To bound $|\mathcal K_2(h)|$, note that
\[
|\mathcal K_2(h)|\log 3
\le \sum_{k\in\mathcal K_2(h)}\log\frac{\det(\Lambda_h^{k+1})}{\det(\Lambda_h^k)}
\le \log\frac{\det(\Lambda_h^{K+1})}{\det(\Lambda_h^1)}
\le d\log\Bigl(1+\frac{MT}{d\lambda}\Bigr).
\]
Therefore,
\begin{equation}
\sum_{k\in\mathcal K_2(h)}\sum_{t=\tau_k}^{\tau_{k+1}-2}\sum_{m=1}^M
\|\phi(x_h^{m,t},a_h^{m,t})\|_{(\Lambda_h^k)^{-1}}
\le
\frac{Md}{\log 3}\sqrt{\gamma\log\gamma}\,\log\Bigl(1+\frac{MT}{d\lambda}\Bigr).
\label{eq:k2-sum-bound}
\end{equation}

\paragraph{Step 4: Analysis for residual last-episode terms.}
For the remaining terms at $t=\tau_{k+1}-1$ with $k \in \mathcal K_2(h)$, we simply use
\[
\|\phi(x,a)\|_{(\Lambda_h^k)^{-1}}\le \frac{1}{\sqrt\lambda},
\]
which follows from $\Lambda_h^k\succeq \lambda I_d$ and $\|\phi(x,a)\|_2\le 1$.
Hence
\begin{equation}
\sum_{k\in\mathcal K_2(h)}\sum_{m=1}^M
\|\phi(x_h^{m,\tau_{k+1}-1},a_h^{m,\tau_{k+1}-1})\|_{(\Lambda_h^k)^{-1}}
\le
\frac{M}{\sqrt\lambda}\,|\mathcal K_2(h)|
\le
\frac{Md}{\sqrt\lambda\,\log 3}\log\Bigl(1+\frac{MT}{d\lambda}\Bigr).
\label{eq:k2-last-step-bound}
\end{equation}

\paragraph{Step 5: combine and sum over steps.}
Adding the bounds \eqref{eq:k1-sum-bound}, \eqref{eq:k2-sum-bound}, \eqref{eq:k2-last-step-bound} and then summing over $h=1,\ldots,H$ proves the lemma.
\end{proof}


\section{Proof of \texorpdfstring{\Cref{theo_hetero}}{Theorem 2}}\label{app:hetero_small}

Throughout this section, we consider the misspecified setting in Definition~\ref{ass:small_heterogeneity}. 
Define $\{(r_{h}, \mathbb{P}_{h})\}_{h=1}^H$ to be the common reference linear reward and transition kernel in \Cref{ass:small_heterogeneity}, where
\[
r_{h}(x,a)=\langle \phi(x,a), \theta_h \rangle,\qquad
\mathbb{P}_{h}(B \mid x,a)=\langle \phi(x,a), \mu_h(B) \rangle.
\]

For each agent $m\in[M]$, we write $Q_{m,h}^{\pi}$ and $V_{m,h}^{\pi}$ for the value functions induced by the local MDP of agent $m$ under policy $\pi$, and $Q_{m,h}^{\star},V_{m,h}^{\star}$ for the corresponding optimal functions under the optimal policy $\pi_m^\star$.

\subsection{Heterogeneity-Specific Lemmas}
\label{heterolemma}
\begin{lemma}
\label{lem:hetero_reference_realization}
For any agent $m\in[M]$, policy $\pi$, and step $h\in[H]$, define
\[
w_{m,h}^{\pi}
:=
\theta_h+\int_{\mathcal X}V_{m,h+1}^{\pi}(x')\,\mu_h(\mathrm d x').
\]
Then, for every $(x,a)\in\mathcal X\times\mathcal A$,
\[
\bigl|Q_{m,h}^{\pi}(x,a)-\phi(x,a)^\top w_{m,h}^{\pi}\bigr|
\le H\xi.
\]
Moreover, $\|w_{m,h}^{\pi}\|_2\le 2H\sqrt d$.
\end{lemma}

\begin{proof}[Proof of Lemma \ref{lem:hetero_reference_realization}]
By the linearity of the reference model,
\[
\phi(x,a)^\top w_{m,h}^{\pi}
=
r_{h}(x,a)+\mathbb P_{h}V_{m,h+1}^{\pi}(x,a).
\]
Therefore,
\[
\begin{aligned}
\bigl|Q_{m,h}^{\pi}(x,a)-\phi(x,a)^\top w_{m,h}^{\pi}\bigr|
&\le
\bigl|r_{m,h}(x,a)-r_{h}(x,a)\bigr|+
\bigl|\bigl(\mathbb P_{m,h}-\mathbb P_{h}\bigr)V_{m,h+1}^{\pi}(x,a)\bigr|
\\
&\le \xi+(H-h)\xi\le H\xi,
\end{aligned}
\]
where we used $0\le V_{m,h+1}^{\pi}\le H-h$. The norm bound follows from
\[
\|w_{m,h}^{\pi}\|_2
\le \|\theta_h\|_2+
\Bigl\|\int_{\mathcal X}V_{m,h+1}^{\pi}(x')\,\mu_h(\mathrm d x')\Bigr\|_2
\le \sqrt d+H\sqrt d
\le 2H\sqrt d.
\]
\end{proof}

\begin{lemma}
\label{lem:hetero_det_perturb}
Let $\{\varepsilon^{m,t}\}_{m\in[M],\,t\in[T]}$ be any collection satisfying
$|\varepsilon^{m,t}|\le B$ for all $(m,t)$. Then, for every step $h\in[H]$, round $k\in[K]$, and vector $u\in\mathbb R^d$,
\[
\Biggl|
u^\top (\Lambda_h^{k})^{-1}
\sum_{m=1}^M\sum_{t=1}^{\tau_k-1}
\phi(x_h^{m,t},a_h^{m,t})\,\varepsilon^{m,t}
\Biggr|
\le
B\sqrt{dM\tau_k}\,\|u\|_{(\Lambda_h^k)^{-1}}.
\]
\end{lemma}

\begin{proof}[Proof of Lemma \ref{lem:hetero_det_perturb}]
Using the triangle inequality,
\[
\begin{aligned}
&\Biggl|
u^\top (\Lambda_h^{k})^{-1}
\sum_{m=1}^M\sum_{t=1}^{\tau_k-1}
\phi(x_h^{m,t},a_h^{m,t})\,\varepsilon^{m,t}
\Biggr|
\\
&\le
B\sum_{m=1}^M\sum_{t=1}^{\tau_k-1}
\bigl|u^\top (\Lambda_h^k)^{-1}\phi(x_h^{m,t},a_h^{m,t})\bigr|
\\
&\le B\sum_{m=1}^M\sum_{t=1}^{\tau_k-1}
\Bigl(\|u\|_{(\Lambda_h^k)^{-1}}\sqrt{\phi(x_h^{m,t},a_h^{m,t})^\top(\Lambda_h^k)^{-1}\phi(x_h^{m,t},a_h^{m,t})}\Bigr)
\\
&\le
B\sqrt{M(\tau_k-1)}\,\|u\|_{(\Lambda_h^k)^{-1}}
\Biggl[
\sum_{m=1}^M\sum_{t=1}^{\tau_k-1}
\phi(x_h^{m,t},a_h^{m,t})^\top(\Lambda_h^k)^{-1}\phi(x_h^{m,t},a_h^{m,t})
\Biggr]^{1/2}.
\end{aligned}
\]
The last two inequalities hold due to the Cauchy--Schwarz inequality. 
Moreover, note that
\[
\begin{aligned}
\sum_{m=1}^M\sum_{t=1}^{\tau_k-1}
\phi(x_h^{m,t},a_h^{m,t})^\top(\Lambda_h^k)^{-1}\phi(x_h^{m,t},a_h^{m,t})
&=
\mathrm{tr}\!\Bigl((\Lambda_h^k)^{-1}(\Lambda_h^k-\lambda I_d)\Bigr)\le d.
\end{aligned}
\]
Combining the two displays and using $M(\tau_k-1)\le M\tau_k$ proves the claim.
\end{proof}

\begin{lemma}[Uniform concentration with agent-dependent transitions]
\label{lem:hetero_vk_self_normalized}
Fix $\bar p\in(0,1)$ and set $\lambda=1$. Suppose the optimistic estimates are formed with bonus parameter
\[
\beta
=
c_\beta^\prime Hd\sqrt{\log\Bigl(1+\frac{dMHT}{\bar p}\Bigr)}
+2H\xi\sqrt{dMT},
\]
where $c_\beta^\prime > 0$ is a sufficiently large positive constant.
Then there exists an absolute constant $C>0$ such that, if we define an event $\mathcal C_{\mathrm{het}}$ as, simultaneously for all steps $h\in[H]$ and rounds $k\in[K]$,
\begin{align*}
    \Biggl\|
\sum_{m=1}^M\sum_{t=1}^{\tau_k-1}
\phi(x_h^{m,t},a_h^{m,t})
&\Bigl[
V_{h+1}^{k}(x_{h+1}^{m,t})-
\mathbb P_{m,h}V_{h+1}^{k}(x_h^{m,t},a_h^{m,t})
\Bigr]
\Biggr\|_{(\Lambda_h^k)^{-1}}
\\
&\quad \le CHd\sqrt{\log\Bigl(1+\frac{(c_\beta^\prime+1)dMHT}{\bar p}\Bigr)}.
\end{align*}
 Then $\mathbb P(\mathcal{C}_{\mathrm{het}})\geq 1-\bar p$.
\end{lemma}

\begin{proof}[Proof of Lemma \ref{lem:hetero_vk_self_normalized}]
The proof follows the same argument as in Lemma~\ref{lem:Vk_self_normalized}. The only differences are the choice of $\beta$ and the replacement of the common transition kernel $\mathbb P_h$ by the agent-dependent transition kernel $\mathbb P_{m,h}$. Neither difference affects the proof, since the argument only relies on boundedness and the martingale property.


\end{proof}

\begin{lemma}
\label{lem:hetero_confidence}
Assume $\lambda=1$. There exists an absolute constant $c_\beta^\prime>0$ such that, for every confidence level $\bar p\in(0,1)$, if
\[
\beta
=
c_\beta^\prime Hd\sqrt{\log\Bigl(1+\frac{dMHT}{\bar p}\Bigr)}
+2H\xi\sqrt{dMT},
\]
then conditioned on the event $\mathcal C_{\mathrm{het}}$ defined in~\Cref{lem:hetero_vk_self_normalized}, simultaneously for all agents $m\in[M]$, rounds $k\in[K]$, steps $h\in[H]$, policies $\pi$, and state-action pairs $(x,a)$,
\[
\Bigl|
\phi(x,a)^\top w_h^k
-Q_{m,h}^{\pi}(x,a)
-\mathbb P_{h}\bigl(V_{h+1}^{k}-V_{m,h+1}^{\pi}\bigr)(x,a)
\Bigr|
\le
\beta\,\|\phi(x,a)\|_{(\Lambda_h^k)^{-1}}+H\xi.
\]
\end{lemma}

\begin{proof}[Proof of Lemma \ref{lem:hetero_confidence}]
Fix $m,k,h,\pi$, and $(x,a)$. Define
\[
\nu_{m,h}^{k,\pi}
:=
\int_{\mathcal X}\bigl(V_{h+1}^{k}(x')-V_{m,h+1}^{\pi}(x')\bigr)\,\mu_h(\mathrm d x').
\]
Then
\begin{equation}
\label{pvpi}
    \mathbb P_{h}\bigl(V_{h+1}^{k}-V_{m,h+1}^{\pi}\bigr)(x,a)
=
\phi(x,a)^\top \nu_{m,h}^{k,\pi}.
\end{equation}
In the misspecified setting, we have
\[
\begin{aligned}
w_h^k&=(\Lambda_h^k)^{-1}\sum_{n=1}^M\sum_{t=1}^{\tau_k-1}\phi(x_h^{n,t},a_h^{n,t})\bigl[r_{n,h}(x_h^{n,t},a_h^{n,t})+V_{h+1}^{k}(x_{h+1}^{n,t})\bigr]\\
&=
(\Lambda_h^k)^{-1}
\sum_{n=1}^M\sum_{t=1}^{\tau_k-1}
\phi(x_h^{n,t},a_h^{n,t})
\Bigl[
\Delta_{n,t,h}^{k}+\eta_{n,t,h}^{k}
+r_{h}(x_h^{n,t},a_h^{n,t})
+\mathbb P_{h}V_{h+1}^k(x_h^{n,t},a_h^{n,t})
\Bigr],
\end{aligned}
\]
where
\[
\Delta_{n,t,h}^{k}
:=
\bigl[r_{n,h}-r_{\mathrm{ref},h}\bigr](x_h^{n,t},a_h^{n,t})
+\bigl(\mathbb P_{n,h}-\mathbb P_{h}\bigr)V_{h+1}^{k}(x_h^{n,t},a_h^{n,t}),
\]
\[
\eta_{n,t,h}^{k}
:=
V_{h+1}^{k}(x_{h+1}^{n,t})-\mathbb P_{n,h}V_{h+1}^{k}(x_h^{n,t},a_h^{n,t}).
\]
According to the definition of $w_{m,h}^{\pi}$ and $\nu_{m,h}^{k,\pi}$, we have
\[
r_{h}(x,a)+\mathbb P_{h}V_{h+1}^k(x,a)
=
\phi(x,a)^\top\bigl(w_{m,h}^{\pi}+\nu_{m,h}^{k,\pi}\bigr),
\]
and then together with \eqref{pvpi},

\begin{align}
&\phi(x,a)^\top w_h^k
-\phi(x,a)^\top w_{m,h}^{\pi}
-\mathbb P_{h}\bigl(V_{h+1}^{k}-V_{m,h+1}^{\pi}\bigr)(x,a)
\notag\\
&=
\phi(x,a)^\top(\Lambda_h^k)^{-1}
\sum_{n=1}^M\sum_{t=1}^{\tau_k-1}
\phi(x_h^{n,t},a_h^{n,t})\,\Delta_{n,t,h}^{k}+
\phi(x,a)^\top(\Lambda_h^k)^{-1}
\sum_{n=1}^M\sum_{t=1}^{\tau_k-1}
\phi(x_h^{n,t},a_h^{n,t})\,\eta_{n,t,h}^{k}
\notag\\
&\quad + \phi(x,a)^\top(\Lambda_h^k)^{-1}
\sum_{n=1}^M\sum_{t=1}^{\tau_k-1}
\phi(x_h^{n,t},a_h^{n,t})\phi(x_h^{n,t},a_h^{n,t})^\top\bigl(w_{m,h}^{\pi}+\nu_{m,h}^{k,\pi}\bigr)\notag\\
&\quad -
\phi(x,a)^\top\bigl(w_{m,h}^{\pi}+\nu_{m,h}^{k,\pi}\bigr)\notag\\
&=
\phi(x,a)^\top(\Lambda_h^k)^{-1}
\sum_{n=1}^M\sum_{t=1}^{\tau_k-1}
\phi(x_h^{n,t},a_h^{n,t})\,\Delta_{n,t,h}^{k}+
\phi(x,a)^\top(\Lambda_h^k)^{-1}
\sum_{n=1}^M\sum_{t=1}^{\tau_k-1}
\phi(x_h^{n,t},a_h^{n,t})\,\eta_{n,t,h}^{k}
\notag\\
&\quad-
\lambda\,\phi(x,a)^\top(\Lambda_h^k)^{-1}\bigl(w_{m,h}^{\pi}+\nu_{m,h}^{k,\pi}\bigr). \label{diff}
\end{align}

In the last equation, we use
$$\sum_{n=1}^M\sum_{t=1}^{\tau_k-1}
\phi(x_h^{n,t},a_h^{n,t})\phi(x_h^{n,t},a_h^{n,t})^\top = \Lambda_h^k - \lambda I_d.$$
For the first term in \eqref{diff},~\Cref{ass:small_heterogeneity} and $0\le V_{h+1}^k\le H$ imply
$|\Delta_{n,t,h}^{k}|\le 2H\xi$, so Lemma~\ref{lem:hetero_det_perturb} gives
\[
\Biggl|
\phi(x,a)^\top(\Lambda_h^k)^{-1}
\sum_{n=1}^M\sum_{t=1}^{\tau_k-1}
\phi(x_h^{n,t},a_h^{n,t})\,\Delta_{n,t,h}^{k}
\Biggr|
\le
2H\xi\sqrt{dMT}\,\|\phi(x,a)\|_{(\Lambda_h^k)^{-1}}.
\]
For the second martingale term in \eqref{diff}, Cauchy-Schwarz inequality and Lemma~\ref{lem:hetero_vk_self_normalized} yields
\[
\begin{aligned}
&\Biggl|
\phi(x,a)^\top(\Lambda_h^k)^{-1}
\sum_{n=1}^M\sum_{t=1}^{\tau_k-1}
\phi(x_h^{n,t},a_h^{n,t})\,\eta_{n,t,h}^{k}
\Biggr|\\
&\le
\|\phi(x,a)\|_{(\Lambda_h^k)^{-1}}
\Biggl\|
\sum_{n=1}^M\sum_{t=1}^{\tau_k-1}
\phi(x_h^{n,t},a_h^{n,t})\,\eta_{n,t,h}^{k}
\Biggr\|_{(\Lambda_h^k)^{-1}}
\\
&\le
CHd\sqrt{\log\Bigl(1+\frac{(c_\beta+1)dMHT}{\bar p}\Bigr)}\,
\|\phi(x,a)\|_{(\Lambda_h^k)^{-1}}.
\end{aligned}
\]
Moreover, for the last term in \eqref{diff}
\[
\begin{aligned}
\bigl\|w_{m,h}^{\pi}+\nu_{m,h}^{k,\pi}\bigr\|_2
&=
\Bigl\|\theta_h+\int_{\mathcal X}V_{h+1}^{k}(x')\,\mu_h(\mathrm d x')\Bigr\|_2
\le
\sqrt d+H\sqrt d
\le 2H\sqrt d,
\end{aligned}
\]
so by Cauchy-Schwarz inequality and $\Lambda_h^k \succeq \lambda I$,
\begin{align*}
    \lambda\,\bigl|\phi(x,a)^\top(\Lambda_h^k)^{-1}(w_{m,h}^{\pi}+\nu_{m,h}^{k,\pi})\bigr|
&\le
\lambda\,\|\phi(x,a)\|_{(\Lambda_h^k)^{-1}}
\|w_{m,h}^{\pi}+\nu_{m,h}^{k,\pi}\|_{(\Lambda_h^k)^{-1}}\\
&\le
2H\sqrt{d\lambda}\,\|\phi(x,a)\|_{(\Lambda_h^k)^{-1}}.
\end{align*}
Finally, Lemma~\ref{lem:hetero_reference_realization} gives
\[
\bigl|\phi(x,a)^\top w_{m,h}^{\pi}-Q_{m,h}^{\pi}(x,a)\bigr|\le H\xi.
\]
Combining this with the preceding upper bounds for the three terms in \eqref{diff}, and using $\tau_k \le T$ and $\lambda = 1$, we obtain
\[
\begin{aligned}
&\Bigl|
\phi(x,a)^\top w_h^k
-Q_{m,h}^{\pi}(x,a)
-\mathbb P_{h}\bigl(V_{h+1}^{k}-V_{m,h+1}^{\pi}\bigr)(x,a)
\Bigr|
\\
&\le
\Biggl(
CHd\sqrt{\log\Bigl(1+\frac{(c_\beta^\prime+1)dMHT}{\bar p}\Bigr)}
+2H\xi\sqrt{dMT}
+2H\sqrt d
\Biggr)
\|\phi(x,a)\|_{(\Lambda_h^k)^{-1}}+H\xi.
\end{aligned}
\]
Choosing $c_\beta^\prime$ sufficiently large absorbs the first and third terms into the displayed value of $\beta$, which proves the claim.
\end{proof}

\begin{lemma}
\label{lem:hetero_biased_optimism}
Under the event $\mathcal C_{\mathrm{het}}$, for every agent $m\in[M]$, round $k\in[K]$, step $h\in[H]$, and state-action pair $(x,a)$,
\[
Q_h^k(x,a)\ge Q_{m,h}^{\star}(x,a)-H(H+1-h)\xi,
\quad
V_h^k(x)\ge V_{m,h}^{\star}(x)-H(H+1-h)\xi.
\]
\end{lemma}

\begin{proof}[Proof of Lemma \ref{lem:hetero_biased_optimism}]
We argue by backward induction on $h$.

For $h=H$, we have $V_{H+1}^k\equiv V_{m,H+1}^{\star}\equiv 0$. Applying Lemma~\ref{lem:hetero_confidence} with $\pi=\pi_m^{\star}$ gives
\[
\phi(x,a)^\top w_H^k
\ge
Q_{m,H}^{\star}(x,a)-\beta\,\|\phi(x,a)\|_{(\Lambda_H^k)^{-1}}-H\xi.
\]
Hence, by the definition of $Q_H^k$,
\[
Q_H^k(x,a)=\Pi_{[0,H]}\Bigl(\phi(x,a)^\top w_H^k+\beta\,\|\phi(x,a)\|_{(\Lambda_H^k)^{-1}}\Bigr)
\ge
Q_{m,H}^{\star}(x,a)-H\xi.
\]
and thus
$$ V_H^k(x) =\max\limits_a Q^k_H(x,a)  \ge \max\limits_a Q^*_{m,H}(x,a)-H\xi = V_{m,H}^*(x)-H\xi.$$
Now assume the claim holds at step $h+1$. Then, by induction, for all $(x,k) \in \mathcal{X} \times[K]$, we have
\[
V_{h+1}^{k}(x)
\ge
V_{m,h+1}^{\star}(x)-H(H-h)\xi.
\]
Therefore,
\[
\mathbb P_{h}\bigl(V_{h+1}^{k}-V_{m,h+1}^{\star}\bigr)(x,a)
\ge -H(H-h)\xi.
\]
Applying Lemma~\ref{lem:hetero_confidence} with $\pi=\pi_m^{\star}$ yields
\[
\phi(x,a)^\top w_h^k
\ge
Q_{m,h}^{\star}(x,a)-H(H-h)\xi-H\xi-\beta\,\|\phi(x,a)\|_{(\Lambda_h^k)^{-1}}.
\]
Thus,
\[
Q_h^k(x,a) =\Pi_{[0,H]}\Bigl(\phi(x,a)^\top w_h^k+\beta\,\|\phi(x,a)\|_{(\Lambda_h^k)^{-1}}\Bigr)
\ge
Q_{m,h}^{\star}(x,a)-H(H+1-h)\xi.
\]
and thus
$$ V_h^k(x) =\max\limits_a Q^k_h(x,a)  \ge \max\limits_a Q^*_{m,h}(x,a)-H(H+1-h)\xi = V_{m,h}^*(x)-H(H+1-h)\xi.$$
This completes the proof for step $h$, and hence the conclusion follows by backward induction.
\end{proof}

For the trajectory generated by agent $m$ in episode $t$, define
\begin{align}
\delta_h^{m,t}
&:=
V_h^{k(t)}(x_h^{m,t})-V_{m,h}^{\pi^{k(t)}}(x_h^{m,t}),
\nonumber\\
\mathcal G_{m,t,h}
&:=
\sigma\!\Big(
\{(x_\ell^{n,j},a_\ell^{n,j})\}_{(n,j)<(m,t),\,\ell\in[H]}
\cup
\{(x_\ell^{m,t},a_\ell^{m,t})\}_{\ell\le h}
\Big),
\nonumber\\
\zeta_{h+1}^{m,t}
&:=
\mathbb E\bigl[\delta_{h+1}^{m,t}\mid \mathcal G_{m,t,h}\bigr]-\delta_{h+1}^{m,t}.
\nonumber
\end{align}
Equivalently,
\[
\zeta_{h+1}^{m,t}
=
\mathbb P_{m,h}\bigl(V_{h+1}^{k(t)}-V_{m,h+1}^{\pi^{k(t)}}\bigr)(x_h^{m,t},a_h^{m,t})
-\bigl(V_{h+1}^{k(t)}-V_{m,h+1}^{\pi^{k(t)}}\bigr)(x_{h+1}^{m,t}).
\]

\begin{lemma}
\label{lem:hetero_one_step}
Conditioned on $\mathcal C_{\mathrm{het}}$, for every agent $m\in[M]$, episode $t\in[T]$, and step $h\in[H]$,
\[
\delta_h^{m,t}
\le
\delta_{h+1}^{m,t}+\zeta_{h+1}^{m,t}
+2\beta\,\|\phi(x_h^{m,t},a_h^{m,t})\|_{(\Lambda_h^{k(t)})^{-1}}
+3H\xi.
\]
\end{lemma}

\begin{proof}[Proof of Lemma \ref{lem:hetero_one_step}]
Fix $m,t$. Let $k = k(t)$ be the round index that the $t$-th episode is in. Since $\pi^{k(t)}$ is greedy with respect to $Q_h^{k(t)}$,
\[
V_h^{k(t)}(x_h^{m,t})=Q_h^{k(t)}(x_h^{m,t},a_h^{m,t}),
\qquad
V_{m,h}^{\pi^{k(t)}}(x_h^{m,t})=Q_{m,h}^{\pi^{k(t)}}(x_h^{m,t},a_h^{m,t}).
\]
Hence
\[
\delta_h^{m,t}
=
Q_h^{k(t)}(x_h^{m,t},a_h^{m,t})-Q_{m,h}^{\pi^{k(t)}}(x_h^{m,t},a_h^{m,t}).
\]
Under the event $\mathcal C_{\mathrm{het}}$, Lemma~\ref{lem:hetero_confidence} implies
\begin{equation}
\label{deltahet}
    \delta_h^{m,t}
\le
\mathbb P_{h}\bigl(V_{h+1}^{k(t)}-V_{m,h+1}^{\pi^{k(t)}}\bigr)(x_h^{m,t},a_h^{m,t})
+2\beta\,\|\phi(x_h^{m,t},a_h^{m,t})\|_{(\Lambda_h^{k})^{-1}}
+H\xi.
\end{equation}
Now compare the reference and local transition operators. For any $(x,a)$,
\[
\begin{aligned}
&\mathbb P_{h}\bigl(V_{h+1}^{k(t)}-V_{m,h+1}^{\pi^{k(t)}}\bigr)(x,a)
-\mathbb P_{m,h}\bigl(V_{h+1}^{k(t)}-V_{m,h+1}^{\pi^{k(t)}}\bigr)(x,a)
\\
&=
\bigl(\mathbb P_{h}-\mathbb P_{m,h}\bigr)V_{h+1}^{k(t)}(x,a)
+\bigl(\mathbb P_{m,h}-\mathbb P_{h}\bigr)V_{m,h+1}^{\pi^{k(t)}}(x,a).
\end{aligned}
\]
Since both value functions lie in $[0,H]$, the right-hand side has upper bound at most $2H\xi$. Therefore,
\[
\mathbb P_{h}\bigl(V_{h+1}^{k(t)}-V_{m,h+1}^{\pi^{k(t)}}\bigr)(x_h^{m,t},a_h^{m,t})
\le
\mathbb P_{m,h}\bigl(V_{h+1}^{k(t)}-V_{m,h+1}^{\pi^{k(t)}}\bigr)(x_h^{m,t},a_h^{m,t})+2H\xi.
\]
Substituting this inequality into \eqref{deltahet} gives
\[
\delta_h^{m,t}
\le
\mathbb P_{m,h}\bigl(V_{h+1}^{k(t)}-V_{m,h+1}^{\pi^{k(t)}}\bigr)(x_h^{m,t},a_h^{m,t})
+2\beta\,\|\phi(x_h^{m,t},a_h^{m,t})\|_{(\Lambda_h^{k})^{-1}}
+3H\xi.
\]
The only additional point caused by the lower clipping is the case where the unprojected optimistic score is negative. In this case $Q_h^{k(t)}=0$, so
\[
\begin{aligned}
\delta_h^{m,t}
&=
-Q_{m,h}^{\pi^{k(t)}}(x_h^{m,t},a_h^{m,t})
\\
&=
-r_{m,h}(x_h^{m,t},a_h^{m,t})
-
P_{m,h}V_{m,h+1}^{\pi^{k(t)}}(x_h^{m,t},a_h^{m,t})
\\
&\le
P_{m,h}\bigl(
V_{h+1}^{k(t)}
-
V_{m,h+1}^{\pi^{k(t)}}
\bigr)(x_h^{m,t},a_h^{m,t}).
\end{aligned}
\]
because $r_{m,h}\ge0$ and $V_{h+1}^{k(t)}\ge0$. Hence the same recursion follows, since the bonus term and the heterogeneity error term are nonnegative.

Finally, by noting that
\[
\mathbb P_{m,h}\bigl(V_{h+1}^{k(t)}-V_{m,h+1}^{\pi^{k(t)}}\bigr)(x_h^{m,t},a_h^{m,t})
=
\mathbb E\bigl[\delta_{h+1}^{m,t}\mid\mathcal G_{m,t,h}\bigr]
=
\delta_{h+1}^{m,t}+\zeta_{h+1}^{m,t},
\]
we finish the proof of the recursion.
\end{proof}

\subsection{Proof of Theorem~\ref{theo_hetero}}

\begin{proof}[Proof of Theorem~\ref{theo_hetero}]
We first note that similar to Lemma~\ref{lem:5}, the following event $\mathcal{E}_{\mathrm{het}}$ holds with probability at least $1-\bar{p}$:
\begin{equation}
\label{hetmartingale}
    \sum_{m=1}^M\sum_{t=1}^T\sum_{h=1}^H\zeta_{h+1}^{m,t}
\le
\sqrt{8H^3MT\log\Bigl(\frac{2}{\bar p}\Bigr)}.
\end{equation}
Then the event $\mathcal C_{\mathrm{het}}\cap\mathcal E_{\mathrm{het}}$ holds with probability at least $1-2\bar p$. We now prove \Cref{theo_hetero} under the event $\mathcal C_{\mathrm{het}}\cap\mathcal E_{\mathrm{het}}$, under which all lemmas in \Cref{heterolemma} hold.

Lemma~\ref{lem:hetero_biased_optimism} with $h=1$ gives
\[
V_{m,1}^{\star}(x_1^{m,t})
\le
V_1^{k(t)}(x_1^{m,t})+H^2\xi.
\]
Hence
\begin{align}
\mathrm{Regret}(T)
&=
\sum_{m=1}^M\sum_{t=1}^T
\Bigl(V_{m,1}^{\star}(x_1^{m,t})-V_{m,1}^{\pi^{k(t)}}(x_1^{m,t})\Bigr)
\notag\\
&\le
\sum_{m=1}^M\sum_{t=1}^T
\Bigl(V_1^{k(t)}(x_1^{m,t})-V_{m,1}^{\pi^{k(t)}}(x_1^{m,t})\Bigr)
+H^2MT\xi
\notag\\
&=
\sum_{m=1}^M\sum_{t=1}^T \delta_1^{m,t}+H^2MT\xi. \label{regretdelta}
\end{align}
Next, summing Lemma~\ref{lem:hetero_one_step} over $h=1,\ldots,H$ and using $\delta_{H+1}^{m,t}=0$ yields
\[
\delta_1^{m,t}
\le
\sum_{h=1}^H\zeta_{h+1}^{m,t}
+2\beta\sum_{h=1}^H\|\phi(x_h^{m,t},a_h^{m,t})\|_{(\Lambda_h^{k(t)})^{-1}}
+3H^2\xi.
\]
Summing over all $(m,t)$ and regrouping by rounds,
\[
\begin{aligned}
\sum_{m=1}^M\sum_{t=1}^T\delta_1^{m,t}
\le
\sum_{m=1}^M\sum_{t=1}^T\sum_{h=1}^H\zeta_{h+1}^{m,t}
+
2\beta
\sum_{h=1}^H\sum_{k=1}^K\sum_{t=\tau_k}^{\tau_{k+1}-1}\sum_{m=1}^M
\|\phi(x_h^{m,t},a_h^{m,t})\|_{(\Lambda_h^k)^{-1}}
+3H^2MT\xi.
\end{aligned}
\]
Therefore, together with \eqref{regretdelta}, we have
\[
\begin{aligned}
\mathrm{Regret}(T)
\le
&\sum_{m=1}^M\sum_{t=1}^T\sum_{h=1}^H\zeta_{h+1}^{m,t}+
2\beta
\sum_{h=1}^H\sum_{k=1}^K\sum_{t=\tau_k}^{\tau_{k+1}-1}\sum_{m=1}^M
\|\phi(x_h^{m,t},a_h^{m,t})\|_{(\Lambda_h^k)^{-1}}
+4H^2MT\xi.
\end{aligned}
\]
Applying the martingale bound in \eqref{hetmartingale} and Lemma~\ref{lem:6} (which is unchanged because it depends only on the Gram matrices and the synchronization rule) gives
\[
\begin{aligned}
\mathrm{Regret}(T)
\le
&
\sqrt{8H^3MT\log\Bigl(\frac{4}{\bar{p}}\Bigr)}+4H^2MT\xi
\\
&+
2\beta H\Biggl(
\sqrt{6dMT\log\Bigl(1+\frac{MT}{d}\Bigr)}
+\frac{Md}{\log 3}(\sqrt{\gamma\log\gamma}+1)\,\log\Bigl(1+\frac{MT}{d}\Bigr)
\Biggr).
\end{aligned}
\]
Substituting
\[
\beta
=
c_\beta^\prime Hd\sqrt{\log\Bigl(1+\frac{2dMHT}{\bar{p}}\Bigr)}
+2H\xi\sqrt{dMT}
\]
and collecting terms yield
\begin{align}
&\mathrm{Regret}(T)
\le\;
\bigl(2c_\beta^\prime\sqrt{6}+4\sqrt{2}\bigr)
\sqrt{
M d^{3} H^{4} T\,
\log\Bigl(1+\frac{2dMHT}{\bar{p}}\Bigr)\,
\log\Bigl(1+\frac{MT}{d}\Bigr)
}
\notag\\
&\quad
+\frac{2c_\beta' M d^{2} H^{2}}{\log 3}\bigl(\sqrt{\gamma\log\gamma}+1\bigr)\,
\sqrt{\log\Bigl(1+\frac{2dMHT}{\bar{p}}\Bigr)}\,
\,\log\Bigl(1+\frac{MT}{d}\Bigr)
\notag\\
&\quad
+4(\sqrt{6}+1)\,\xi MdH^{2}T\,
\sqrt{\log\Bigl(1+\frac{MT}{d}\Bigr)}
+\frac{4\xi\sqrt{M^3d^3H^4T}}{\log 3}\,
\bigl(\sqrt{\gamma\log\gamma}+1\bigr)\,
\log\Bigl(1+\frac{MT}{d}\Bigr). \label{finalregretmis}
\end{align}
which is exactly the claimed bound in \Cref{theo_hetero}.
\end{proof}

Extending our communication-efficient federated RL guarantees in the misspecified setting to broader heterogeneity and participation models, such as stronger distributional shifts across agents or asynchronous and partial client participation, remains an interesting direction for future work.

\section{Proof of \texorpdfstring{\Cref{theo_2}}{Theorem 3}}
\label{app:proof_comm}
\begin{proof}[Proof of \Cref{theo_2}]

Fix any $\alpha\in(0,\gamma)$ and divide the rounds into three classes:
\[
I_{\mathrm{long}}\coloneqq\{k:\tau_{k+1}-\tau_k>\alpha\},
\]
\[
I_{\mathrm{short,tr}}\coloneqq\left\{k:\tau_{k+1}-\tau_k\le \alpha\ \text{and round $k$ ends because the determinant trigger fires}\right\},
\]
\[
I_{\mathrm{short,term}}\coloneqq\left\{k:\tau_{k+1}-\tau_k\le \alpha\ \text{and round $k$ ends only because the hard cap $T$ is reached}\right\}.
\]
By definition, these three sets form a partition of $[K]$.
Moreover, $|I_{\mathrm{short,term}}|\le 1$, because only the final round can terminate solely due to the hard stop.

\paragraph{Long rounds.}
For each $k\in I_{\mathrm{long}}$ we have $\tau_{k+1}-\tau_k>\alpha$.
Summing the round lengths and using that the total number of episodes per agent is $T$ gives
\[
\alpha |I_{\mathrm{long}}|
\le \sum_{k\in I_{\mathrm{long}}}(\tau_{k+1}-\tau_k)
\le \sum_{k=1}^{K}(\tau_{k+1}-\tau_k)
= T,
\]
so
\begin{equation}
|I_{\mathrm{long}}|\le \frac{T}{\alpha}.
\label{eq:I-long-bound}
\end{equation}

\paragraph{Short rounds that end by a determinant trigger.}
Fix any $k\in I_{\mathrm{short,tr}}$.
Then some agent $m$ and some step $h$ satisfy the trigger condition at the end of round $k$:
\[
\frac{\det(\Lambda_h^k+\Lambda_h^{m,k,\mathrm{loc}})}{\det(\Lambda_h^k)}
\ge \frac{\gamma}{\tau_{k+1}-\tau_k}
\ge \frac{\gamma}{\alpha}.
\]
Since the global update dominates the single-agent update in the Loewner order,
\[
\Lambda_h^{k+1}=\Lambda_h^k+\sum_{m=1}^M\Lambda_h^{m,k,\mathrm{loc}}\succeq \Lambda_h^k+\Lambda_h^{m,k,\mathrm{loc}},
\]
and the determinant is monotone over positive semidefinite matrices. Therefore,
\[
\frac{\det(\Lambda_h^{k+1})}{\det(\Lambda_h^k)}
\ge \frac{\det(\Lambda_h^k+\Lambda_h^{m,k,\mathrm{loc}})}{\det(\Lambda_h^k)}
\ge \frac{\gamma}{\alpha}.
\]
For every other step $\ell\ne h$, we still have $\det(\Lambda_\ell^{k+1})/\det(\Lambda_\ell^k)\ge 1$.
Hence
\[
\prod_{\ell=1}^H \frac{\det(\Lambda_\ell^{k+1})}{\det(\Lambda_\ell^k)}\ge \frac{\gamma}{\alpha}.
\]
Multiplying over all $k\in I_{\mathrm{short,tr}}$ yields
\begin{equation}
    \label{alphagamma}
\Bigl(\frac{\gamma}{\alpha}\Bigr)^{|I_{\mathrm{short,tr}}|}
\le
\prod_{k\in I_{\mathrm{short,tr}}}\prod_{\ell=1}^H \frac{\det(\Lambda_\ell^{k+1})}{\det(\Lambda_\ell^k)}
\le
\prod_{\ell=1}^H\frac{\det(\Lambda_\ell^{K+1})}{\det(\Lambda_\ell^1)}.
\end{equation}
By the arithmetic-geometric mean inequality applied to the eigenvalues of $\lambda^{-1}\Lambda_\ell^{K+1}$,
\begin{equation}
\label{eigen}
    \det(\Lambda_\ell^{K+1})
\le \lambda^d\Bigl(\frac{\mathrm{tr}(\Lambda_\ell^{K+1})}{\lambda d}\Bigr)^d
\le \lambda^d\Bigl(1+\frac{MT}{\lambda d}\Bigr)^d.
\end{equation}
The last inequality holds due to the definition in \eqref{lastlambda} with $\|\phi(x_h^{m,t},a_h^{m,t})\|_2 \leq 1$. Since $\Lambda_\ell^1=\lambda I_d$, we further obtain
\[
\frac{\det(\Lambda_\ell^{K+1})}{\det(\Lambda_\ell^1)}
\le \Bigl(1+\frac{MT}{\lambda d}\Bigr)^d.
\]
Consequently, back to \eqref{alphagamma}, we have
\begin{equation}
|I_{\mathrm{short,tr}}|\le \frac{dH}{\log(\gamma/\alpha)}\log\Bigl(1+\frac{MT}{\lambda d}\Bigr).
\label{eq:I-short-tr-bound}
\end{equation}

\paragraph{Put the three classes together.}
Combining the previous bounds  \eqref{eq:I-long-bound} and \eqref{eq:I-short-tr-bound} gives
\[
K = |I_{\mathrm{long}}|+|I_{\mathrm{short,tr}}|+|I_{\mathrm{short,term}}|
\le \frac{T}{\alpha}+\frac{dH}{\log(\gamma/\alpha)}\log\Bigl(1+\frac{MT}{\lambda d}\Bigr)+1.
\]
Choosing $\alpha=\gamma/2$ yields the claimed bound
\[
K\le 1+\frac{2T}{\gamma}+\frac{dH}{\log 2}\log\Bigl(1+\frac{MT}{d\lambda}\Bigr).
\]
\end{proof}

\end{document}